\documentclass[a4paper]{amsart} 

\usepackage{amsmath,amsthm,amssymb,amsfonts,mathrsfs,color,hyperref, mathtools,crop, graphicx, enumitem, todonotes}
\usepackage[]{geometry}

\theoremstyle{plain}
\begingroup
\newtheorem{theorem}{Theorem}[section]
\newtheorem*{theorem*}{Theorem}
\newtheorem*{"theorem"}{``Theorem''}
\newtheorem{corollary}[theorem]{Corollary}

\newtheorem{lemma}[theorem]{Lemma}
\endgroup

\theoremstyle{definition}
\begingroup
\newtheorem{definition}[theorem]{Definition}
\endgroup

\theoremstyle{remark}
\begingroup
\newtheorem{remark}[theorem]{Remark}
\newtheorem{example}[theorem]{Example}
\endgroup 

\numberwithin{equation}{section}
\newcommand{\N}{\mathbb N} 
 
\newcommand{\Z}{\mathbb Z} 
\newcommand{\R}{\mathbb R} 
\newcommand{\E}{{\mathbb E}}
\renewcommand{\P}{{\mathbb P}}

\newcommand{\dist}{{\rm dist}}

\newcommand{\spt}{{\mathrm{spt}}}

\renewcommand{\H}{{\mathcal H}}

\renewcommand{\L}{{\mathcal L}}

\newcommand{\F}{{\mathcal F}}

\newcommand{\Risk}{\mathcal{R}}

\newcommand{\LRa} {\Leftrightarrow}
\newcommand{\Ra} {\Rightarrow}

\renewcommand{\d}{\mathrm{d}}

\newcommand{\dx}{\,\mathrm{d}x}
\newcommand{\dy}{\,\mathrm{d}y}
\newcommand{\dw}{\,\mathrm{d}w}
\newcommand{\dz}{\,\mathrm{d}z}

\newcommand{\eps}{\varepsilon}
\newcommand{\average}{{\mathchoice {\kern1ex\vcenter{\hrule height.4pt
width 6pt depth0pt} \kern-9.7pt} {\kern1ex\vcenter{\hrule
height.4pt width 4.3pt depth0pt} \kern-7pt} {} {} }}

\newcommand{\Y}{\mathcal{Y}^\circ}

\DeclareMathOperator*{\argmin}{argmin} 
 
\newcommand{\Rad}{\mathrm{Rad}}
\newcommand{\B}{\mathcal{B}}
\newcommand\showlabel{\addtocounter{equation}{1}\tag{\theequation}}

\allowdisplaybreaks
 
 \makeatletter
\@namedef{subjclassname@2020}{2020 Mathematics Subject Classification}
\makeatother

\begin{document}

\title[A priori estimates for classification problems]{A priori estimates for classification problems using neural networks}

\author{Weinan E}
\address{Weinan E\\
Program for Applied and Computational Mathematics and Department of Mathematics\\
Princeton University\\
Princeton, NJ 08544
}
\email{weinan@math.princeton.edu}

\author{Stephan Wojtowytsch}
\address{Stephan Wojtowytsch\\
Program for Applied and Computational Mathematics\\
Princeton University\\
Princeton, NJ 08544
}
\email{stephanw@princeton.edu}

\date{\today}


\begin{abstract}
We consider binary and multi-class classification problems using hypothesis classes of neural networks. For a given hypothesis class, we use Rademacher complexity estimates and direct approximation theorems to obtain a priori error estimates for regularized loss functionals. 
\end{abstract}

\maketitle


\section{Introduction}

Many of the most prominent successes of neural networks have been in classification problems, and many benchmark problems for architecture prototypes and optimization algorithms are on data sets for image classification. Despite this situation, theoretical results for classification are scarce compared to the more well-studied field of regression problems and function approximation. In this article, we extend the a priori error estimates of \cite{E:2018ab} for regression problems to classification problems.

Compared with regression problems for which we almost always use square loss, there are several different common loss functions for classification problems, which have slightly different mathematical and geometric properties. For the sake of simplicity, we restrict ourselves to binary classification in the introduction.
\begin{enumerate}
\item Square loss can be used also in classification problems with a target function that takes a discrete set of values on the different classes. Such a function coincides with a Barron function $\P$-almost everywhere if the data distribution is such that the different classes have positive distance. In this setting, regression and classification problems are indistinguishable, and the estimates of \cite{E:2018ab} (in the noiseless case) apply. We therefore focus on different models in this article.

\item More often, only one-sided $L^2$-approximation (square hinge loss) or one-sided $L^1$-approx\-imation (hinge loss) is considered, as we only need a function to be large positive/large negative on the different classes with no specific target value. This setting is similar to $L^2$-approximation since minimizers exist, which leads to basically the same a priori estimates as for $L^2$-regression. 

On the other hand, the setting is different in the fact that minimizers are highly non-unique. In particular, {\em any} function which is sufficiently large and has the correct sign on both classes is a minimizer of these risk functionals. No additional regularity is encoded in the risk functional.

\item Loss functionals of cross-entropy type also encourage functions to be large positive or negative on the different classes, but the loss functions are strictly positive on the whole real line (with exponential tails). This means that minimizers of the loss functional do not exist, causing an additional logarithmic factor in the a priori estimates.

On the other hand, these risk functionals regularize minimizing sequences. In a higher order expansion, it can be seen that loss functions with exponential tails encourage maximum margin behaviour, i.e.\ they prefer the certainty of classification to be as high as possible {\em uniformly} over the different label classes. This is made more precise below.
\end{enumerate}

The article is organized as follows. In Section \ref{section binary}, we study binary classification using abstract hypothesis classes and neural networks with a single hidden layer in the case that the two labelled classes are separated by a positive spatial distance. After a brief discussion of the links between correct classification and risk minimization (including the implicit biases of different risk functionals), we obtain a priori estimates under explicit regularization. In particular, we introduce a notion of classification complexity which takes the role of the path-norm in a priori error estimates compared to regression problems. 

The chief ingredients of our analysis are estimates on the Rademacher complexity of two-layer neural networks with uniformly bounded path-norm and a direct approximation theorem for Barron functions by finite neural networks. We briefly illustrate how these ingredients can be used in other function classes, for example those associated to multi-layer networks or deep residual networks in Section \ref{section multi-layer}. The most serious omission in this article are convolutional neural networks, since we are not aware of a corresponding function space theory.

The analysis is extended to multi-label classification in Section \ref{section multi-label}, and to the case of data sets in which different classes do not have a positive spatial separation in Section \ref{section infinite complexity}. The setting where minimizers do not exist is also studied in the case of general $L^2$-regression in Appendix \ref{appendix regression estimates}.

\subsection{Notation}

We denote the support of a Radon measure $\mu$ by $\spt\,\mu$. If $\phi$ is $\mu$ measurable and locally integrable, we denote by $\phi\cdot\mu$ the measure which has density $\phi$ with respect to $\mu$. If $\Phi:X\to Y$ is measurable and $\mu$ is a measure on $X$, we denote by $\Phi_\sharp \mu$ the push-forward measure on $Y$. If $(X,d)$ is a metric space, and $A, B\subseteq X$, we denote the distance between $A$ and $B$ by $\dist(A,B) = \inf_{x\in A, x'\in B} d(x,x')$ and $\dist(x, A) := \dist(\{x\}, A)$.

All measures on $\R^d$ are assumed to be defined on the Borel $\sigma$-algebra (or the larger $\sigma$-algebra which additionally contains all null sets).

\section{Complexity of binary classification}\label{section binary}

\subsection{Preliminaries}
In this section, we introduce the framework in which we will consider classification problems.

\begin{definition}
A {\em binary classification problem} is a triple $(\P, C_+, C_-)$ where $\P$ is a probability distribution on $\R^d$ and $C_+, C_- \subseteq\R^d$ are disjoint $\P$-measurable sets such that $\P(C_+ \cup C_-) = 1$.
\end{definition}

The {\em category function}
\[
y: \R^d\to\R, \qquad y_x = \begin{cases} 1 &x\in C_+\\ -1 &x\in C_-\\ 0 &\text{else}\end{cases}
\]
is $\P$-measurable.

\begin{definition}
We say that a binary classification problem is {\em solvable in a hypothesis class $\H$} of $\P$-measurable functions if there exists $h\in \H$ such that
\begin{equation}
y_x\cdot h(x) \geq 1 \qquad\P-\text{almost everywhere}.
\end{equation}
\end{definition}

We consider the closure of the classes $C_\pm$, which is slightly technical since the classes are currently only defined in the $\P$-almost everywhere sense. Let
\[
\overline C_+ = \bigcap_{A \in \mathcal C_+}A, \qquad \mathcal C_+ = \left\{A\subseteq \R^d \:\big|\: A\text{ is closed and } y \leq 0 \:\:\P-\text{a.e.\ outside of A}\right\}.
\]
The closure of $C_-$ is defined the same way with $-y$ in place of $y$.

\begin{lemma}\label{lemma finite distance}
Assume that every function in $\H$ is Lipschitz continuous with Lipschitz constant at most $L>0$. If $(\P, C_+, C_-)$ is solvable in $\H$, then
\begin{equation}\label{eq lipschitz bound}
\dist(\overline C_+, \overline C_-) \geq \frac2L.
\end{equation}
\end{lemma}

\begin{proof}
Consider $h\in \H$ such that $h\geq 1$ $\P$-almost everywhere on $C_+$. By continuity, we expect $h\geq 1$ on $\overline C_+$. In the weak setting, we argue as follows: Since $h$ is continuous, we find that $h^{-1}[1,\infty)$ is closed, and since $y \leq 0$ $\P$-almost everywhere outside of $h^{-1}[1,\infty)$, we find that $\overline C_+ \subseteq h^{-1}[1,\infty)$. The same argument holds for $C_-$ and $\overline C_-$.

If $(\P, C_+, C_-)$ is solvable in $\H$, then there exists a function $h\in \H$ such that $y_x\,h(x)\geq 1$ $\P$-almost everywhere, so in particular $h\geq 1$ on $\overline C_+$ and $h\leq -1$ on $\overline C_-$. Since $h$ is $L$-Lipschitz, we find that for any $x\in \overline C_+$ and $x'\in \overline C_-$ we have
\[
2 = \big|h(x) - h(x')\big| \leq L\,|x-x'|.
\]
The result follows by taking the infimum over $x\in \overline C_+$ and $x'\in \overline C_-$. 
\end{proof}

Many hypothesis classes $\H$ are further composed of sub-classes of different complexities (e.g.\ by norm in function spaces or by number of free parameters). In that situation, we can quantify solvability.

\begin{definition}\label{definition complexity classes}
Let $\{\H_Q\}_{Q\in(0,\infty)}$ be a family of hypothesis classes such that $Q < Q' \Ra \H_Q\subseteq \H_{Q'}$. If $\H = \bigcup_{Q>0} \H_Q$, we say that a binary classification problem $(\P, C_1, C_2)$ is {\em solvable in $\H$ with complexity $Q$} if the problem is solvable in $\H_{Q+\eps}$ for every $\eps>0$. We denote
\begin{equation}
Q = Q_\H(\P, C_1, C_2) = \inf \big\{Q'>0 : (\P, C_1, C_2) \text{ is solvable in }\H_{Q'}\big\}.
\end{equation}
\end{definition}

Clearly, a problem is solvable in $\H$ if and only if it is solvable in $\H_Q$ for sufficiently large $Q$, thus if and only if $Q_\H<\infty$.

\begin{example}
If $\H_Q$ is the space of functions which are Lipschitz-continuous with Lipschitz-constant $Q$, then the lower bound \eqref{eq lipschitz bound} is sharp as
\[
h(x) = \frac{\dist(x, \overline C_-) - \dist(x, \overline C_+)}{\delta}
\]
is Lipschitz-continuous with Lipschitz-constant $\leq 2/\delta$ and satisfies $h\geq 1$ on $\overline C_+$, $h\leq -1$ on $\overline C_-$.
\end{example}

\begin{remark}
Assume that $h$ is a hypothesis class of continuous functions. If $\P_n = \frac1n \sum_{i=1}^n\delta_{x_i}$ is the empirical measure of samples drawn from $\P$, then the complexity of $(\P_n, C_-, C_+) $ is lower than that of $(\P, C_-, C_+)$ since a function $h\in \H$ which solves $y_x\cdot h(x)\geq 1$ $\P$-almost everywhere also satisfies $y_x\cdot h(x)\geq 1$ $\P_n$-almost everywhere (with probability $1$ over the sample points $x_1,\dots,x_n$).
\end{remark}

\subsection{Two-layer neural networks}

Let us consider classification using two-layer neural networks (i.e.\ neural networks with a single hidden layer). We consider the ReLU activation function $\sigma(z) = \max\{z,0\}$.

\begin{definition}[Barron space]
Let $\pi$ be a probability measure on $\R^{d+2}$. We set $f_\pi:\R^d\to\R$,
\[
f_\pi(x) = \E_{(a,w,b)\sim \pi} \big[ a\,\sigma(w^Tx+b)\big], \qquad (a,w,b)\in \R\times \R^d\times\R,
\]
assuming the expression is well-defined. We consider the norm
\[
\|f\|_\B = \inf \left\{ \E_{(a,w,b)\sim\pi} \big[|a|\,(|w|+|b|)\big] : f\equiv f_\pi\right\}
\]
and introduce Barron space 
\[
\B = \{f:\R^d\to\R : \|f\|_\B<\infty\}.
\]
\end{definition}

\begin{remark}
For ReLU-activated networks, the infimum in the definition of the norm is attained \cite{barron_new}. To prove this, one can exploit the homogeneity of the activation function and use a compactness theorem in the space of (signed) Radon measures on the sphere.
\end{remark}

\begin{remark}
Like the space of Lipschitz-continuous functions, Barron space itself does not depend on the norm on $\R^d$, but when the norm on $\R^d$ is changed, also the Barron norm is replaced by an equivalent one. Typically, we consider $\R^d$ to be equipped with the $\ell^\infty$-norm. In general, the norm on parameter space ($w$-variables) is always as dual to the one on data space ($x$-variables) such that $|w^Tx| \leq |w|\,|x|$. Usually, this means that the $\ell^1$-norm is considered on parameter space.

The only result in this article which depends on the precise choice of norm on data space is the Rademacher complexity estimate in Lemma \ref{lemma rademacher barron}, which can be traced also into the a priori error estimates. Occasionally in the literature, both data and parameter space are equipped with the Euclidean norm. In that case, the results remain valid and mildly dimension-dependent constants $\log(2d+2)$ can be eliminated. 
\end{remark}

\begin{remark}
Like other function spaces, Barron spaces depend on the domain under consideration. Enforcing the equality $f= f_\pi$ on $\R^d$ is often too rigid and leads to meaningless assignments a large portion of data space. More commonly, we only require that $f= f_\pi$ $\P$-almost everywhere (where $\P$ describes the data distribution) or equivalently $f\equiv f_\pi$ on $\spt\,\P$, e.g.\ in \cite{barron_new,approximationarticle}. This is the smallest sensible definition of the Barron norm, since the infimum is taken over the largest possible class. If $\P$ is a probability distribution and $K$ is a compact set, we may denote
\begin{align*}
\|f\|_{\B(\P)} &= \inf \left\{ \E_{(a,w,b)\sim\pi} \big[|a|\,(|w|+|b|)\big] : f(x) = f_\pi(x) \text{ for $\P$-a.e.\ }x\in\R^d\right\}\\
\|f\|_{\B(K)} &= \inf \left\{ \E_{(a,w,b)\sim\pi} \big[|a|\,(|w|+|b|)\big] : f(x) = f_\pi(x) \text{ for all }x\in K\right\}.
\end{align*}
Always considering $\B(\P)$ leads to the smallest possible constants. Early works on Barron space focused on $\B([0,1]^d)$ to estimate quantities uniformly over unknown distributions $\P$ \cite{E:2018abpub,weinan2019lei}. This uniformity leads to greater convenience. For example, we do not need to distinguish that $f(x) = 1_{\{x>1/2\}}$ is in $\B(\P_1)$ but not $\B(\P_2)$ where $\P_1 = \frac12(\delta_1 + \delta_0)$ and $\P_2$ is the uniform distribution on $(0,1)$.
\end{remark}

The hypothesis class of Barron space is stratified as in Definition \ref{definition complexity classes} by
\[
\B = \bigcup_{R>0} \overline B_R
\]
where $\overline B_R$ denotes the closed ball of radius $R>0$ around the origin in Barron space. We show that if the hypothesis class used for a classification problem is Barron space, a positive distance between classes is not only necessary, but also sufficient for solvability. However, the necessary complexity may be polynomial in the separation and exponential in the dimension.

\begin{theorem}\label{theorem complexity bound}
Assume that $\delta:= \dist(\overline C_+, \overline C_-)>0$ and there exists $R>0$ such that $\spt(\P)\subseteq B_R(0)$. Then there exists a $f\in \B$ such that 
\begin{align*}
\|f\|_{\B} \leq c_d\left(\frac{R+\delta}\delta\right)^d, \qquad fy \equiv 1\quad\P-\text{a.e.}
\end{align*}
where $c_d$ depends only on the volume of the $d$-dimensional unit ball and the properties of a suitable mollifier.
\end{theorem}

\begin{proof}
We define
\[
\tilde y:\R^d\to \R, \qquad \tilde y(x) = \begin{cases} 1 &\text{if }\dist (x, \overline C_+) < \dist (x, \overline C_-)\text{ and }|x|<R+\delta\\ -1 &\text{if }\dist (x, \overline C_-) < \dist (x, \overline C_+)\text{ and }|x|<R+\delta\\ 0 &\text{if }\dist (x, \overline C_+) = \dist (x, \overline C_-)\text{ or }|x|\geq R+\delta\end{cases}, \qquad f = \eta_\delta * \tilde y
\]
where $\eta_\delta$ is a standard mollifier supported on $B_{\delta/2}(0)$. We find that $y\cdot f \equiv 1$ on $\overline C_- \cup \overline C_+$. Additionally, the Fourier transform of $f$ satisfies the inequality
\begin{align*}
\int_{\R^d} |\hat f|\,|\xi|\,\d\xi
	&= \int_{\R^d} |\widehat{\eta_\delta*\tilde y}|\,|\xi|\,\d\xi\\
	&= \int_{\R^d} |\widehat{\eta_\delta}|\,|\widehat{\tilde y}|\,|\xi|\,\d\xi\\
	&\leq \|\widehat{\tilde y}\|_{L^\infty} \int_{\R^d} \left|\delta\,\hat \eta_1\left(\delta\xi\right)\right|\,|\xi|\,\d\xi\\
	&\leq \|\tilde y\|_{L^1(B_{R+\delta})}\,\delta^{-d}\int_{\R^d} \left|\hat \eta_1\left(\delta\xi\right)\right|\,|\delta \xi|\,\delta^d\,\d\xi\\
	&= c_d \left(\frac{R+\delta}\delta\right)^d\,\int_{\R^d}|\hat \eta_1|\,|\xi|\,\d\xi,
\end{align*}
so by Barron's spectral criterion \cite{barron1993universal} we have
\[
\|f\|_\B \leq \int_{\R^d} |\hat f|\,|\xi|\,\d\xi \leq c_d \left(\frac{R+\delta}\delta\right)^d\,\int_{\R^d}|\hat \eta_1|\,|\xi|\,\d\xi.
\]
\end{proof}

The exponential dependence on dimension is a worst case scenario and not expected if we can identify low-dimensional patterns in the classification problem. In particularly simple problems, even the rate obtained in Lemma \ref{lemma finite distance} is optimal.

\begin{example}
If $\overline C_-\subseteq \{x : x_1< -\delta/2\}$ and $\overline C_+\subseteq \{x : x_1> \delta/2\}$, then
\[
f(x) = \frac x\delta = \frac{\sigma(e_1^Tx) - \sigma(-e_1^Tx)}\delta
\]
is a Barron function of norm $2/\delta$ on $\R^d$ which satisfies $y\cdot f\geq 1$ on $\overline C_+\cup \overline C_-$, so $Q_{\B} \leq \frac{2}\delta$.
\end{example}

\begin{example}
Another example with a relatively simple structure are radial distributions. Assume that $C_+= \lambda \cdot S^{d-1}$ and $C_- = \mu\cdot S^{d-1}$ for some $\lambda,\mu>0$. The data distribution $\P$ does not matter, only which sets are $\P$-null. For our convenience, we recognize the radial structure of the problem and equip $\R^d$ with the Euclidean norm, both in the $x$ and $w$ variables.

We assume that $\spt\,\P = C_+\cup C_-$ such that the inequality $y\cdot h \geq 1$ has to be satisfied on the entire set $C_+\cup C_-$. Consider the function
\[
f_{\alpha,\beta}(x) = \alpha + \beta\,|x| = \alpha\,\sigma(0) + c_d\,\beta \int_{S^{d-1}} \sigma(w^Tx)\,\pi^0(\d w)
\]
where $\pi^0$ denotes the uniform distribution on the sphere and $c_d$ is a normalizing factor. We compute that
\[
\int_{S^{d-1}} \sigma(w^Tx)\,\pi^0(\d w) =  \frac{\int_0^1w_1(1-w_1^2)^{d-2}\dw_1}{\int_{-1}^1 (1-w_1^2)^{d-2}\dw_1} |x| = \frac{\frac{1}{2\,\frac{d-2}2+2}}{\frac{\sqrt{\pi}\,\Gamma(\frac{d-2}2+1)}{\Gamma(\frac{d-2}2 + \frac32)}}\,|x| = \frac{\Gamma((d+1)/2)}{\sqrt{\pi}\,d\,\Gamma(d/2)}|x|
\]
The normalizing factor satisfies
\[
c_d = \frac{\sqrt{\pi}\,d\,\Gamma(d/2)}{\Gamma((d+1)/2)} = \sqrt{2\pi d} + O\left(d^{-1/2}\right).
\]
Thus for sufficiently large $d$ we have $\|f_{\alpha,\beta}\|_\B \leq |\alpha| + \sqrt{2\pi d +1} \,|\beta|$ and 
\[
f_{\alpha,\beta} = \alpha + \beta \mu \text{ on } C_-, \qquad f_{\alpha,\beta}\equiv \alpha + \lambda\mu \text{ on }C_+.
\]
Thus $f_{\alpha,\beta}$ satisfies $y\cdot f_{\alpha,\beta}\geq 1$ $\P$-almost everywhere if and only if
\[
\begin{cases} \alpha + \beta\mu &\leq -1\\	\alpha + \beta\lambda &\geq 1\end{cases} \qquad \Ra\qquad \begin{cases}\beta(\lambda-\mu)&\geq 2\\ \alpha + \beta\lambda &\geq 1\end{cases}.
\]
If we assume that $\lambda>\mu$, we recognize that $R= \lambda$ and $\delta= \lambda-\mu$, so we can choose 
\[
\beta = \frac2{\lambda-\mu} = \frac2\delta, \qquad \alpha = 1 - \frac{2\lambda}{\lambda-\mu}= 1 - \frac{2R}\delta
\]
and obtain that 
\begin{align*}
Q_\B(\P, C_+, C_-) \leq 1 + \frac{2R}{\delta} + \frac{4\,\sqrt{2\pi d+1}}\delta.
\end{align*}
A mild dimension-dependence is observed, but only in the constant, not the dependence on $\delta$.
 If the data symmetry is cylindrical instead (i.e. $C_\pm = \lambda_\pm \cdot S^{k-1} \times \R^{d-k}$), the constant $\sqrt{d}$ can be lowered to $\sqrt{k}$ by considering a Barron function which only depends on the first $k$ coordinates.
\end{example}

\begin{remark}
The fact that cylindrical structures can be recognized with lower norm than radial structures is an advantage of shallow neural networks over isotropic random feature models. 

If the data distribution is radial, but the classes are alternating (e.g.\ $C_- = S^{d-1}\cup 3\cdot S^{d-1}$ and $C_+ = 2\cdot S^{d-1}$), it may be useful to use a neural network with at least two hidden layers, the first of which only needs to output the radial function $|x|$. A related observation concerning a regression problem can be found in \cite[Remark 5.9]{barron_new}. Classification using deeper neural networks is discussed briefly in Section \ref{section multi-layer}.
\end{remark}

{
On the other hand, in general the complexity of a classification problem may be exponentially large in dimension.
}

\begin{example}
For $\delta=\frac1N$, let $\P= \frac1{(2N+1)^d} \sum_{x\in \{-N,\dots,N\}^d} \delta_{x/N}$ be the uniform distribution on the grid $(\delta\Z)^d$ inside the hypercube $[-1,1]^d$. In particular, $\P$ is a probability distribution on $[-1,1]^d$. Since any two points in the grid on which $\P$ is supported have distance at least $\delta = 1/N$, we find that for any of the $(2N+1)^d$ binary partitions of $\spt\,\P$, the distance between the classes is $\delta$.
We use Rademacher complexity to show that at least one of these partitions has classification complexity at least
\begin{equation}
Q\geq \left(\frac{2+\delta}{\delta}\right)^{d/2}
\end{equation}
in Barron space. A more thorough introduction to Rademacher complexity is given below in Definition \ref{definition rademacher} and Lemma \ref{lemma rademacher barron} for further background, but we present the proof here in the more natural context.

Let $Q$ be such that for any category function $y:\spt\,\P\to \{-1,1\}$, there exists $h^*$ such that $h^*(x)\cdot y_x\geq 1$ for all $x\in\spt\,\P$ and $\|h^*\|_\B\leq Q$. As usual, $\H_Q$ denotes the ball of radius $Q>0$ in Barron space. Furthermore, we denote by $\psi:\R\to\R$ the $1$-Lipschitz function which is $-1$ on $(-\infty,-1]$ and $1$ on $[1,\infty)$. By assumption, the hypothesis class $\F_Q = \{\psi\circ h : h\in \H_Q\}$ coincides with the class $\F^*$ of all function $h:\spt\,\P\to \{-1,1\}$. Finally, we abbreviate $n_d = (2N+1)^d$ and $\spt\,\P = \{x_1,\dots, x_{n_d} \}$. If $\xi$ are iid random variables which take values $\pm 1$ with equal probability, then
\begin{align*}
1 &= \E_{\xi} \left[\sup_{f\in \F^*} \frac1{n_d}\sum_{i=1}^{n_d}\xi_i\,f(x_i)\right]\\
	&= \Rad (\F^*, \spt\,\P)\\
	&= \Rad(\F_Q, \spt\,\P)\\
	&\leq \Rad(\H_Q, \spt\,\P)\\
	&\leq \frac{Q}{\sqrt{n_d}},
\end{align*}
where we used the contraction lemma for Rademacher complexities, \cite[Lemma 26.9]{shalev2014understanding}.

\end{example}

A particularly convenient property in Barron space is the fact that a finitely parametrized subset has immense approximation power, as formalized below. The scalar-valued $L^2$-version of the following result was proved originally in \cite{E:2018ab} and a weaker version in \cite{barron1993universal}. The proof goes through also in the Hilbert space-valued case and with the slightly smaller constant claimed below.

\begin{theorem}[Direct Approximation Theorem, $L^2$-version]\label{direct approximation theorem barron}
Let $f^* \in \B^k$ be a vector-valued Barron function, i.e.\
\[
f^*(x) = \E_{(a,w,b)\sim\pi} \big[a\,\sigma(w^Tx+b)\big]
\]
where $\pi$ is a parameter distribution on $\R^k\times\R^d\times \R$ and 
\[
\|f^*\|_\B = \E_{(a,w,b)\sim \pi}\big[|a|_{\ell^2}\,\big(|w|+|b|\big)\big].
\]
Let $\P$ be a probability distribution such that $\spt(\P)\subseteq [-R,R]^d$. Then there exists $f_m (x)= \frac1m \sum_{i=1}^m a_i\,\sigma(w_i^Tx+b_i)$ such that
\begin{equation}\label{eq lp approximation}
\frac1m \sum_{i=1}^m |a_i|_{\ell^2}\big[|w_i|+|b_i|\big] \leq \|f^*\|_\B, \qquad \left\|f_m - f^*\right\|_{L^2(\P)} \leq \frac{\|f^*\|_\B\,\max\{1,R\}}{\sqrt{m}}.
\end{equation}
\end{theorem}

A similar result holds in the uniform topology on any compact set. A proof in the scalar-valued case can be found in \cite[Theorem 12]{review_article}, see also \cite[Remark 13]{review_article}. The scalar version is applied component-wise below.

\begin{theorem}[Direct Approximation Theorem, $L^\infty$-version]\label{direct approximation theorem uniform}
Let $K\subset[-R,R]^d$ be a compact set in $\R^d$ and $f^*\in \B^k$ a vector-valued Barron function.
Under the same conditions as above, there exists a two-layer neural network with $km$ parameters $\tilde f_m (x)= \frac1{km} \sum_{i=1}^{km} \tilde a_i\,\sigma(\tilde w_i^Tx+\tilde b_i)$ such that
\begin{equation}\label{eq linfty approximation}
\frac1{km} \sum_{i=1}^{km} |a_i|_{\ell^2}\big[|w_i|+|b_i|\big] \leq \|f^*\|_\B,\qquad \|f_m - f^*\|_{C^0(K)} \leq \|f^*\|_\B \max\{1,R\}\,\sqrt{\frac{k\,(d+1)}m}.
\end{equation}
\end{theorem}

\begin{remark}
Theorem \ref{direct approximation theorem uniform} admits a slight dimension-dependent improvement to $\|f_m-f^*\|_{L^\infty} \leq C\,\frac{\log(m)}{m^{1/2 + 1/d}}$ at the expense of a less explicit constant \cite{makovoz1998uniform}.
\end{remark}

In the context of classification problems, this gives us the following immediate application.

\begin{corollary}[Mostly correct classification and correct classification]\label{corollary mostly correct}
Let $m\in \N$ and $(\P, C_+, C_-)$ a binary classification problem which can be solved in Barron space with complexity $Q>0$. If $\spt(\P)\subseteq[-R,R]^d$, there exists a two-layer neural network $h_m$ with $m$ neurons such that
\begin{align}
\P\left(\left\{x\in \R^d : y_x\cdot h_m(x) < 0\right\}\right)& \leq \frac{Q^2\,\max\{1,R\}^2}m,\\
 \P\left(\left\{x\in \R^d : \eps_x\cdot h_m(x) < 1/2\right\}\right) &\leq \frac{4\,Q^2\,\max\{1,R\}^2}m.
\end{align}
Furthermore, if $m\geq Q^2\,\max\{1,R\}^2 (d+1)$, there exists a two-layer neural network with $m$ neurons such that $\P\left(\left\{x\in \R^d : y_x\cdot h_m(x) < 0\right\}\right) = 0$.
\end{corollary}

\begin{proof}
Let $h\in \B$ such that $h\geq 1$ on $\overline C_+$ and $h\leq -1$ on $\overline C_-$. Set $\Omega =  \overline C_+\cup \overline C_-$ and recall that $\P(\R^d\setminus \Omega) =0$. Let $h_m$ be like in the direct approximation theorem. By Chebysheff's inequality we have
\[
\P\left(\big\{x\in \Omega : y_xh_m(x)<0\big\}\right) \leq \P\left(\big\{x\in\Omega : |h_m(x) - h(x)|>1\right)\leq \frac{\int_\Omega |h_m-h|^2\,\P(\d x)}{1} \leq \frac{Q^2\,\max\{1,R\}^2}m.
\]
The second inequality is proved analogolously. The correct classification result is proved using the $L^\infty$-version of the direct approximation theorem.
\end{proof}

\begin{remark}
The complexity $Q$ of a binary classification problem in $[-1,1]^d$ is a priori bounded by $C_d\,\delta^{-d}$, where $\delta$ is the spatial separation between the two classes, $d$ is the dimension of the embedding space, and $C_d$ is a constant depending only on $d$. While this can be prohibitively large, all such problems are `good' in the sense that we only need to double the number of parameters in order to cut the probability of the misclassified set in half in the a priori bound. While the constant of classification complexity is affected, the rate is not. This is different from the classical curse of dimensionality, where relationship between the number of parameters $m$ and the error $\varepsilon$ would take a form like $\varepsilon \sim m^{-\alpha/d}$ for $\alpha\ll d$ instead of $\varepsilon \sim C\,m^{-1}$ with a potentially large constant. 
\end{remark}

\subsection{Classification and risk minimization}\label{section margin}

Typically, we approach classification by minimizing a risk functional $\Risk$. We focus on a specific type of loss functional
\[
\Risk(h) = \int_{\R^d} L\big(-y_x\,h(x)\big)\,\P(\d x)
\]
where again $\P$ is the data distribution and $y$ is the category function as above. The {\em loss function} $L$ is assumed to be monotone increasing and sufficiently smooth, facilitating alignment between $y$ and $h$. Furthermore, we assume that $\lim_{z\to -\infty} L(z) = 0$ (which is equivalent to the assumption that $L$ is lower-bounded, up to a meaningless translation). 

Risk minimization can be seen as a proxy for classification due to the following observation:
\begin{equation}\label{eq risk is proxy}
\P(\{x:y_xh(x)< 0\}) \leq \int_{\{x:y_xh(x)< 0\}} \frac{L(-y_xh(x))}{L(0)}\,\P(\d x) \leq \frac1{L(0)}\int_{\R^d} L\big(-y_x\,h(x)\big)\,\P(\d x) \leq \frac{\Risk(h)}{L(0)}.
\end{equation}

The most pressing question when minimizing a functional is whether minimizers exist. 

\begin{remark}
In \eqref{eq risk is proxy}, we used that $L\geq 0$. If $L$ is not bounded from below, the loss functional may not encourage correct classification since classification with high certainty on parts of the domain may compensate incorrect classification in others.
\end{remark}

\begin{remark}
An interesting dynamic connection between risk minimization and correct classification is explored in \cite{Berlyand:2020aa} where it is shown that (under conditions) for every $\eps>0$ there exists $\delta>0$ such that the following holds: if at time $t_0\geq 0$ a set of probability at least $1-\delta$ is classified correctly, and network weights are trained using gradient descent, then at all later times, a set of probability at least $1-\eps$ is classified correctly.
\end{remark}

\begin{example}[Hinge-loss]
Let $L(y) = \max\{0, 1+y\}$. Then in particular $L(-y_xh(x)) =0$ if and only if $y_xh(x) \geq 1$. Thus any function $h$ which satisfies $y_xh(x)\geq 1$ is a minimizer of $\Risk$. Since any function which correctly classifies the data is also a risk minimizer (up to rescaling), no additional geometric regularity is imposed. In particular, a minimizer which transitions close to one class instead of well between the classes is just as competitive as a minimizer which gives both classes some margin. This may make minimizers of hinge loss more vulnerable to so-called {\em adversarial examples} where small perturbations to the data lead to classification errors. 

In this situation, we may search for a minimizer of the risk functional with minimal Barron norm to obtain stable classification. Such a minimizer may be found by including an explicit regularizing term as a penalty in the risk functional. This is the approach taken in this article, where we obtain a priori estimates for such functionals. Another approach is that of implicit regularization, which aims to show that certain training algorithms find low norm minimizers. This requires a precise study of training algorithm and initialization.

We note that unlike loss functions with exponential tails (see below), hinge-loss does not have implicit regularization properties encoded in the risk functional.
\end{example}

With different loss functions, minimizers generally do not exist. This is a mathematical inconvenience, but can introduce additional regularity in the problem as we shall see.

\begin{lemma}
Assume that the following conditions are met:
\begin{enumerate}
\item $L(z)>0$ for all $z\in\R$ and $\inf_{z\in\R}L(z) = \lim_{z\to -\infty}L(z) = 0$.
\item The hypothesis class $\H$ is a cone which contains a function $h$ such that $y\cdot h < 0$ $\P$-almost everywhere.
\end{enumerate}
Then $\Risk$ does not have any minimizers.
\end{lemma}

\begin{proof}
Note that $L(-\lambda\,y_x\,h(x))$ is monotone decreasing in $\lambda$, so by the monotone convergence theorem we have
\[
\lim_{\lambda\to \infty} \Risk(\lambda h) = \int_{\R^d} \lim_{\lambda\to \infty} L\big(\lambda\,y_x\,h(x)\big)\,\P(\d x) = L(-\infty) = 0.
\]
On the other hand, $L(y)>0$ for all $y\in \R$, so a minimizer cannot exist.
\end{proof}

When considering minimizing sequences for such a risk functional, we expect the norm to blow up along the sequence. This motivates us to study the `profile at infinity' of such risk functionals by separating magnitude and shape. Let $\H$ be the closed unit ball in a hypothesis space of Lipschitz-continuous function (e.g.\ Barron space). For simplicity, assume that $L$ is strictly monotone increasing and continuous, thus continuously invertible. We define the functionals
\[
\Risk_\lambda:\H\to (0, \infty), \qquad {\Risk}_\lambda(h) = \Risk(\lambda h) = \int_{\R^d} L\big(-\lambda\,y_x\cdot h(x)\big)\,\P(\d x).
\]
To understand what the limit of $\Risk_\lambda$ is, it is convenient to normalize the functionals as 
\[
L^{-1}\left( \int_{\R^d} L\big(-y_x\cdot h(x)\big)\,\P(\d x)\right).
\]
Since $L$ is strictly monotone increasing, we immediately find that 
\begin{align*}
\lambda\,\min_{x\in\spt(\P)}\big(-y_x\cdot h(x)\big)
&= \min_{x\in\spt(\P)}\big(-y_x\cdot \lambda h(x)\big)\\
	&= L^{-1}\left( \int_{\R^d} L\left(\min_{x\in\spt(\P)}-\lambda\,y_x\cdot h(x)\right)\,\P(\d x)\right)\\
	&\leq L^{-1}\left( \int_{\R^d} L\big(-\lambda\, y_x\cdot h(x)\big)\,\P(\d x)\right)\\\showlabel
	&= L^{-1}\left( \int_{\R^d} L\left(\max_{x\in\spt(\P)}-y_x\cdot \lambda h(x)\right)\,\P(\d x)\right)\\
	&\leq \max_{x\in\spt(\P)}\big(-\lambda \, y_x\cdot h(x)\big)\\
	&= \lambda\,\max_{x\in\spt(\P)}\big(-y_x\cdot h(x)\big).
\end{align*}
Thus to consider the limit, the correct normalization is 
\[
\F_\lambda: \H\to (0,\infty), \qquad \F_\lambda(h) = \frac{L^{-1}\big(\Risk_\lambda(h)\big)}\lambda = \frac1\lambda\,L^{-1}\left( \int_{\R^d} L\big(-\lambda\,y_x\cdot h(x)\big)\,\P(\d x)\right).
\]
Note the following:
\begin{equation}
h\in \argmin_{h'\in \bar B_1}\F_\lambda(h')\qquad\LRa \qquad \lambda h\in \argmin_{h'\in \bar B_\lambda} \Risk(h').
\end{equation}
Thus, if $\F_\lambda$ converges to a limiting functional $\F_\infty$ in the sense of $\Gamma$-convergence (such that minimizers of $\F_\lambda$ converge to minimizers of $\F_\infty$), then studying $\F_\infty$ describes the behavior that minimizers of $\Risk$ have for very large norms. We call $\F_\infty$ the {\em margin functional} of $\Risk$.

\begin{example}[Power law loss]
Assume that $y_x\cdot h(x)\geq \eps$ for $\P$-almost all $x$ for some $\eps>0$ and that $L(y) = |y|^{-\beta}$ for some $\beta>0$ and all $y\leq -\mu<0$. Then
\begin{align*}
\F_\lambda(h) &= -\frac1\lambda\left( \int_{\R^d} \big|-\lambda\,y_x\cdot h(x)\big|^{-\beta}\,\P(\d x)\right)^{-\frac1\beta}\\
	&=- \left( \int_{\R^d} \big|-y_x\cdot h(x)\big|^{-\beta}\,\P(\d x)\right)^{-\frac1\beta}
\end{align*}
for all $\lambda>\frac{\mu}\eps$. Thus if $L$ decays to $0$ at an algebraic rate, the {margin functional}
\begin{equation}
\F_\infty = \lim_{\lambda\to \infty} \F_\lambda = -\left( \int_{\R^d} \big|y_x\cdot h(x)\big|^{-\beta}\,\P(\d x)\right)^{-\frac1\beta}
\end{equation}
is an $L^p$-norm for negative $p$ (and in particular, an integrated margin functional). The limit is attained e.g.\ in the topology of uniform convergence (since the sequence is constant). In particular, the loss functional and the margin functional coincide due to homogeneity. 
\end{example}

Taking $\beta\to \infty$ in a second step, we recover the integrated margin functionals converge to the more classical {\em maximum margin functional}. While $\|f\|_{L^p}\to\|f\|_{L^\infty}$ as $p\to \infty$, the $L^\beta$-``norm'' for large negative $\beta$ approaches the value closest to zero $\|1/f\|_{L^\infty}$< i.e.
\[
\lim_{\beta\to\infty}\left(- \int_{\R^d} \big|y_x\cdot f_{\pi}(x)\big|^{-\beta}\,\P(\d x)\right)^{-\frac1\beta} = -\min_{x\in\spt\P} y_x\cdot f_{\pi}(x).
\]
Minimizing $\F_\infty$ corresponds to maximizing $y_x\cdot h(x)$ uniformly over $\spt\,\P$ under a constraint on the size of . Similar behaviour is attained with exponential tails.

\begin{example}[Exponential loss]
Assume that either
\begin{enumerate}
\item $y_x\cdot h(x)\leq 0$ for $\P$-almost all $x$ and that $L(y) = \exp(y)$ for all $y\leq 0$ or
\item $L(y) = \exp(y)$ for all $y\in \R$.
\end{enumerate}
Then
\[
\lim_{\lambda\to\infty} \F_\lambda(h) = \lim_{\lambda\to \infty}\left[\frac1\lambda \log\left(\int_{\R^d}\exp\left(-\lambda\,y_x\,h(x)\right)\,\P(\d x)\right)\right] = \max_{x\in\spt\P}\big(-y_x\cdot h(x)\big) = -\min_{x\in\spt\,\P}(y_x\,h(x)).
\]
Thus the limiting functional here is the {\em maximum margin functional}. The pointwise limit may be improved under reasonable conditions. More generally, we prove the following.
\end{example}

\begin{lemma}\label{convergence lemma}\label{lemma exponential tail margin}
Let $L$ be a function such that for every $\eps>0$ the functions
\[
g_{\lambda}(z) = \frac{L^{-1}(\eps\,L(\lambda z))}\lambda
\]
converge to $g_\infty(z) = z$ locally uniformly on $\R$ (or on $(0,\infty)$).
Then $\F_\infty(h) = \min_{x\in\spt\,\P}y_x\cdot h(x)$ is the maximum margin functional (at functions which classify all points correctly if the second condition is imposed). If there exists a uniform neighbourhood growth function
\begin{equation}\label{eq uniform growth}
\rho:(0,\infty)\to (0,\infty), \qquad \P\big(B_r(x)\big)\geq \rho(r)\quad\forall\ r>0, \:x\in \spt(\P)
\end{equation}
then the limit is uniform (and in particular in the sense of $\Gamma$-convergence).
\end{lemma}

\begin{remark}
The same result can be proven in any function class which has a uniform modulus of continuity (e.g.\ functions whose $\alpha$-H\"older constant is uniformly bounded for some $\alpha>0$).

The uniform neighbourhood growth condition holds for example when $\P$ has a reasonable density with respect to Lebesgue measure or the natural measure on a smooth submanifold of $\R^d$, when $\P$ is an empirical measure, and even when $\P$ is the natural distribution on a self-similar fractal. A reasonable density is for example one whose decay to zero admits some lower bound at the edge of its support.
\end{remark}

\begin{remark}
The exponential function satisfies the uniform convergence condition on $\R$ since
\[
\frac{L^{-1}(\eps\,L(\lambda z))}\lambda = \frac{\log(\eps\,\exp(\lambda z))}\lambda = \frac{\log(\eps)}\lambda + z.
\]
The function $L(z) = \log(1+\exp(z))$ satisfies the uniform convergence condition on $(-\infty, 0)$ since
\[
L^{-1}(z) = \log(e^z-1), \qquad L^{-1}(\eps\,L(z)) = \log\left(e^{\eps\,\log(1+e^{\lambda z})}-1\right) = \log\left((1+e^{\lambda z})^\eps-1\right)
\]
as $\lambda\to \infty$, $e^{-\lambda z}$ becomes uniformly close to $0$ when $z$ is bounded away from zero and the first order Taylor expansion
\[
L^{-1}(\eps\,L(z))\approx \log(\eps\,e^{-\lambda z}) = \log(\eps)+\lambda z
\]
becomes asymptotically valid.
This function has analytic advantages over exponential loss since the exponential tail of the loss function is preserved, but the function is globally Lipschitz-continuous.
\end{remark}

\begin{proof}[Proof of Lemma \ref{convergence lemma}]
Recall that $\F_\lambda(h) \leq \max_{x'\in \spt\P}\big[-y_{x'}h(x')\big]$ independently of $L$, so only the inverse inequality needs to be proved.

We observe that $\max_{x'\in \spt\P}[-y_{x'}h(x')]$ is bounded due to the Barron norm bound, so the locally uniform convergence holds. Let $r>0$. Let $\bar x$ be a point where $\max_{x'\in \spt\,\P}[-y_{x'}h(x')]$ is attained. Since all $h\in \H$ are $1$-Lipschitz, we find that $-y_xh(x) \geq \max_{x'\in \spt\P}\big[-y_{x'}h(x')\big] - r$ for all $x\in B_r(\bar x)$. In particular
\begin{align*}
\F_\lambda(h) &= \frac1\lambda\,L^{-1}\left( \int_{\R^d} L\big(-\lambda\,y_x\cdot h(x)\big)\,\P(\d x)\right)\\
	&\geq \frac1\lambda\,L^{-1}\left( \int_{B_r(\bar x)} L\big(-\lambda\,y_x\cdot h(x)\big)\,\P(\d x)\right)\\
	&\geq \frac1\lambda L^{-1}\left(\rho(r)\, L\left(\lambda \max_{x'\in \spt\P}\big[-y_{x'}h(x')\big] - \lambda r\right)\right)\\
	&\geq g_\infty\big(-y_{x'}h(x') -r \big) - e_{\lambda,r}
\end{align*}
where $\lim_{\lambda\to \infty} e_{\lambda, r} = 0$. Thus taking $\lambda$ to infinity first with a lower limit and $r\to0$ subsequently, the theorem is proved.
\end{proof}

The connection between risk minimization and classification hardness thus is as follows: When minimizing $\Risk$ in a hypothesis space space, elements of a minimizing sequence will (after normalization) resemble minimizers of the margin functional $\F_\infty$. If $\F_\infty$ is the maximum margin functional, then 
\[
\min_{h\in \H} \F_\infty(h) = \min_{h\in \H}\max_{x\in\spt\P}\big(-\,y_x\cdot h(x)\big) = -\max_{h\in \H} \min_{x\in\spt\P}\big(y_x\cdot h(x)\big)
\]
i.e.\ the minimizer of $\F_\infty$ maximizes $\min_{x\in\spt\P}\big(y_x\cdot h(x)\big)$ in the unit ball of the hypothesis space. Assuming that 
\[
\max_{h\in \H}\min_{x\in\spt\P}\big(y_x\cdot h(x)\big) >0
\]
we obtain that
\begin{align*}
Q_{\B}(\P, C_+, C_-) \leq Q &\quad  \LRa\quad \exists h\in \overline{B_Q(0)}\subseteq \B\text{ s.t. } \min_{x\in \overline C_+ \cup \overline C_-} y_xh(x)\geq 1\\
	&\quad  \LRa \quad \exists h\in\overline{B_1(0)}\subseteq \B\text{ s.t. } \min_{x\in \overline C_+ \cup \overline C_-} y_xh(x)\geq \frac1Q\\
	&\quad  \LRa \quad \min_{h\in \H} \F_\infty(h) \leq - \frac1Q
\end{align*}
where the dependence on $\P, C_+, C_-$ on the right hand side is implicit in the risk functional by the choice of $\overline C_\pm$ for the margin functional. In particular, exponential loss encodes a regularity margin property which is not present when using hinge loss. 

\begin{remark}
Chizat and Bach show that properly initialized gradient descent training finds maximum margin solutions of classification problems, if it converges to a limit at all \cite{Chizat:2020aa}. The result is interpreted as an implicit regularization result for the training mechanism. We maintain that it is the margin functional instead which drives the shapes to maximum margin configurations, and that gradient flow merely follows the asymptotic energy landscape. We remark that the compatibility result is highly non-trivial since the limit described here is only of $C^0$-type and since the map from the weight of a network to its realization is non-linear.
\end{remark}

\subsection{A priori estimates}\label{section a priori}

In this section, we show that there exist finite neural networks of low risk, with an explicit estimate in the number of neurons. Furthermore, we prove a priori error estimates for regularized empirical loss functionals.

\begin{lemma}[Functions of low risk: Lipschitz loss]\label{lemma low risk Lipschitz}
Assume that $h^*\in \B$ and $L$ has Lipschitz constant $[L]$. Then for every $m\in \N$ there exists a two-layer network $h_m$ with $m$ neurons such that 
\[
\|h_m\|_\B\leq \|h^*\|_\B\qquad\text{and}\qquad
\Risk(h_m) \leq \Risk(h^*) +  \frac{[L]\cdot \|h^*\|_\B\,\max\{1,R\}}{\sqrt{m}}.
\]
In particular, if $(\P, C_+, C_-)$ has complexity $\leq Q$ in Barron space, then for any $\lambda>0$ there exists a two-layer network $h_m$ with $m$ neurons such that
\[\showlabel
\|h_m\|_{\B}\leq \lambda Q, \qquad \Risk(h_m) \leq L(-\lambda) + \frac{[L]\cdot\lambda Q\,\max\{1,R\}}{\sqrt{m}}
\]
\end{lemma}

\begin{proof}
Let $h_m$ such that $\|h_m\|\leq \|h^*\|_\B$ and $\|h^*-h_m\|_{L^1} \leq \frac{\|h^*\|_\B\max\{1,R\}}{\sqrt m}$. Then
\begin{align*}
\Risk(h_m) &= \int_{\R^d} L(- y_xh_m(x))\,\P(\d x)\\
	&= \int_{\R^d}L\big(y_xh^*(x) + y_x [h^*(x) -h_m(x)]\big)\,\P(\d x)\\
	&\leq\int_{\R^d} L\big(- y_xh^*(x)\big) + \big|h^*(x) -h_m(x)\big|\,\P(\d x)\\
	&= \Risk(h^*) + \|h^*-h_m\|_{L^1(\P)}.
\end{align*}
For the second claim consider $\lambda h^*$ where $h^*\in \B$ such that $\|h^*\|_\B=Q$ and $y_x\cdot h^*(x)\geq 1$. 
\end{proof}

The rate can be improved for more regular loss functions.

\begin{lemma}[Functions of low risk: Smooth loss]\label{lemma low risk smooth}
Assume that $L\in W^{2,\infty}(\R)$ and for $\lambda>0$ denote $\delta_\lambda:= \max\{|L'|(\xi) : \xi<-\lambda\}$.
Assume further that $(\P, C_+, C_-)$ has complexity $\leq Q$ in Barron space. Then for any $\lambda>0$ there exists a two-layer network $h_m$ with $m$ neurons such that
\[\showlabel
\|h_m\|_{\B}\leq \lambda \big[1+\delta_\lambda\big] Q, \qquad \Risk(h_m) \leq L(-\lambda) +  \left[1 + \delta_\lambda\right]^2\frac{\|L''\|_{L^\infty(Q)}\,(\lambda Q)^2\max\{1,R\}^2}{2m}
\]
\end{lemma}

\begin{proof}
Note that for $z, z^*\in \R$ the identity
\begin{align*}
L(z) - L(z^*) &= L'(z^*)\,(z-z^*) + \big[L(z) - L(z^*) - L'(z^*)\,(z-z^*) \big]\\
	&= L'(z^*)\,(z-z^*) + \big[L(z) - L(z^*) - L'(z^*)\,(z-z^*) \big]\\
	&= L'(z^*)\,(z-z^*)  + \int_{z^*}^z L''(\xi)\,(z-\xi)\d\xi
\end{align*}
holds. Let $h^*\in \B$. For any function $h_m$, note that 
\begin{align*}
\Risk(h_m) - \Risk(h^*) &= \int_{\R^d} L'(y_x\cdot h^*(x))\,y_x(h_m-h^*)(x) \P(\d x) + \int_{\R^d}\int_{y_x\cdot h^*(x)}^{y_x\cdot h_m(x)} L''(\xi)\,(y_x\,h_m(x)-\xi)\d\xi\,\P(\d x)\\
	&\leq \int_{\R^d} L'(y_x\cdot h^*(x))\,y_x(h_m-h^*)(x) \P(\d x) + \frac{\|L''\|_{L^\infty(\R)}}2 \int_{\R^d}\big|h_m- h^*\big|^2(x)\,\P(\d x)
\end{align*}
Now let $\tilde h_m$ such that $\|h_m\|\leq \|h^*\|_\B$ and $\|h^*-h_m\|_{L^2} \leq \frac{\|h^*\|_\B\max\{1,R\}}{\sqrt m}$. Then in particular
\[
|c_m|:= \left|\int_{\R^d} L'(y_x\cdot h^*(x))\,y_x(h_m-h^*)(x) \P(\d x)\right| \leq \max_{x\in\spt\,\P}\big|L'(y_x\cdot h^*(x))\big|\,\frac{\|h^*\|_\B\max\{1,R\}}{\sqrt m}.
\]
Note that $h_m:= \tilde h_m + c_m$ is a two-layer neural network satisfying
\begin{align*}
\|h_m\|_\B &\leq \|h^*\|_\B \left[1 + \frac{\max\{1,R\}\,\max_{x\in\spt\,\P}\big|L'(y_x\cdot h^*(x))\big|}{\sqrt{m}}\right]\\
 \|h_m - h^*\|_{L^2(\P)} &\leq  \|\tilde h_m - h^*\|_{L^2(\P)} + |c_m| \leq \left[1 + \max_{x\in\spt\,\P}\big|L'(y_x\cdot h^*(x))\big|\right]\frac{\|h^*\|_\B\max\{1,R\}}{\sqrt m}.
\end{align*}
In particular
\begin{align*}
\Risk(h_m)
	&\leq \Risk(h^*) + \int_{\R^d} L'(y_x\cdot h^*(x))\,y_x(h_m-h^*)(x) \P(\d x) + \frac{\|L''\|_{L^\infty(\R)}}2 \int_{\R^d}\big|h_m- h^*\big|^2(x)\,\P(\d x)\\
	&\leq \Risk(h^*) + 0 + \left[1 + \frac{\|L''\|_{L^\infty(\R)}}2\max_{x\in\spt\,\P}\big|L'(y_x\cdot h^*(x))\big|\right]^2\frac{\|h^*\|_\B^2\max\{1,R\}^2}{m}.
\end{align*}
The Lemma follows by considering $\lambda h^*$ where $h^*\in \B$ such that $\|h^*\|_\B=Q$ and $y_x\cdot h^*(x)\geq 1$. 
\end{proof}

We use this result to present the a priori error estimate for a regularized model. As we want to derive an estimate with relevance in applications, we need to bound the difference between the (population) risk functional and the empirical risk which we access in practical situations by means of a finite data sample. A convenient tool to bound this discrepancy is Rademacher complexity.

\begin{definition}\label{definition rademacher}
Let $\H$ be a hypothesis class and $S = \{x_1,\dots, x_N\}$ a data set. The Rademacher complexity of $\H$ with respect to $S$ is defined as
\[\showlabel
\Rad(\H; S) = \E_{\xi}\left[\sup_{h\in\H} \frac1N\sum_{i=1}^N \xi_i\,h(x_i)\right],
\]
where $\xi_i$ are iid random variables which take the values $\pm 1$ with probability $1/2$.
\end{definition}

The Rademacher complexity is a convenient tool to decouple the size of oscillations from their sign by introducing additional randomness. It can therefore be used to estimate cancellations and is a key tool in controlling generalization errors. 

\begin{lemma}\cite{bach2017breaking, E:2018ab}\label{lemma rademacher barron}
Let $\H$ be the unit ball in Barron space. Then 
\[
\Rad(\H;S)\leq 2\,\sqrt{\frac{\log(2d+2)}N}
\]
 for any sample set $S\subseteq [-1,1]^d$ with $N$ elements.
\end{lemma}

\begin{remark}
The proof relies on the $\ell^\infty$-norm being used for $x$-variables and thus the $\ell^1$-norm being used for $w$-variables. If the $\ell^2$-norm was used on both hypothesis classes instead, the factor $\log(2d+2)$ could be replaced by a dimension-independent factor $1$.

The proof can easily be modified to show that $\Rad(\H; S)\leq 2\,\max\{1,R\}\,\sqrt{\frac{\log(2d+2)}N}$ if $\spt\P\subseteq [-R,R]^d$. While the bound is not sharp and becomes dimension-independent as $R\to 0$, it does not approach zero since the class of constant functions has non-trivial Rademacher complexity.
\end{remark}

The {\em uniform} bound on the Rademacher complexity of $\H$ with respect to any training set in the unit cube in particular implies a bound on the expectation. 
This estimate is particularly useful for Lipschitz-continuous loss functions. Recall the following result.

\begin{lemma}\cite[Theorem 26.5]{shalev2014understanding}\label{lemma rademacher bound}
Assume that $L(-y_x\cdot h(x))\leq \bar c$ for all $h\in \H$ and $x\in \spt(\P)$. Then with probability at least $1-\delta$ over the choice of set $S\sim \P^n$ we have
\[
\sup_{h\in\H} \left[\int_{\R^d}L(-y_x\cdot h(x))\,\P(\d x) - \frac1n\sum_{i=1}^n L(-y_{x_i}\cdot h(x_i)) \right] \leq 2\,\E_{S'\sim\P^n} \Rad(\H;S') + \bar c\,\sqrt{\frac{2\,\log(2/\delta)}n}.
\]
\end{lemma}

We consider the Barron norm as a regularizer for risk functionals. Note that for a finite neural network $f_m(x) = \frac1m\sum_{i=1}^m a_i\,\sigma(w_i^Tx+b_i)$, the estimates
\[
m\,\|f_m\|_\B \leq \sum_{i=1}^m |a_i|\,\big[|w_i|+|b_i|\big] \leq \frac12 \sum_{i=1}^m \left[ |a_i|^2 + |w_i|^2 + |a_i|^2 + |b_i|^2\right] = \sum_{i=1}^m \left[|a_i|^2 + \frac{|w_i|^2 + |b_i|^2}{2}\right]
\]
hold. In fact, choosing optimal parameters $(a_i, w_i, b_i)$ and scaling such that $|a_i| = |w_i| = |b_i|$ (using the homogeneity of the ReLU function), the left and right hand side can be made equal.

\begin{theorem}[A priori estimate: Hinge loss]\label{theorem a priori hinge}
Let $L(z) = \max\{0, 1+z\}$.
Assume that $\P$ is a data distribution such that $\spt\,\P\subseteq [-R,R]^d$.
Consider the regularized (empirical) risk functional
\[
\widehat \Risk_{n,\lambda} (a, w,b) = \frac1n\sum_{i=1}^n L\big(- y_{x_i}\,f_{(a,w,b)}(x_i)\big) + \frac{\lambda}m\,\sum_{i=1}^m|a_i|\,\big[|w_i| + |b_i|\big].
\]
where $\lambda = \frac{\max\{1,R\}}{\sqrt{m}}$ and $f_{(a,w,b)} = f_m$ with weights $a= (a_i)_{i=1}^m, w = (w_i)_{i=1}^m, (b_i)_{i=1}^m$.
For any  $\delta \in (0, 1)$, with probability at least $1-\delta$ over the choice of iid data points $x_i$ sampled from $\P$, the minimizer $(\hat a, \hat w, \hat b)$ satisfies
\[
\Risk(f_{(\hat a, \hat w, \hat b)}) \leq \frac{2Q\,\max\{1,R\}}{\sqrt{m}} + 2\,\max\{1,R\}\sqrt{\frac{\log(2d+2)}n} + 2Q\,\max\{1,R\}\sqrt{\frac{\log (2/\delta)}n}.\showlabel
\]
Since $L(0)=1$, the estimate \eqref{eq risk is proxy} implies furthermore that 
\[
\P(\{x : y_x\cdot f_{(\hat a, \hat w, \hat b)}(x)<0\}) \leq \Risk(f_{(\hat a, \hat w, \hat b)}).
\]

\end{theorem}

\begin{proof}
The penalty term is nothing else than the Barron norm in a convenient parametrization. We can therefore just treat it as $\|h_m\|_\B$ and work on the level of function spaces instead of the level of parameters.

{\bf Step 1.} Let $h^*$ be a function such that $y_x\cdot h^*(x) \geq 1$ on $\spt\,\P$ and $\|h^*\|_\B\leq Q$. Since the random sample points $x_i$ lie in $\spt\,\P$ with probability $1$, we find that $\Risk_n(h^*) =0$ for any sample. Using Lemma \ref{lemma low risk Lipschitz} for the empirical distribution $\P_n = \frac1n\sum_{i=1}^n \delta_{x_i}$, we know that there exists a network $h_m$ with $m$ neurons such that
\[
\|h_m\|_\B\leq \|h^*\|_\B, \qquad \widehat\Risk_n(h_m) \leq \widehat \Risk_n(h^*) + \frac{\|h^*\|_\B\,\max\{1,R\}}{\sqrt{m}} \leq \frac{Q\,\max\{1,R\}}{\sqrt{m}}.
\] 
Thus
\[
\widehat\Risk_{n,\lambda}(\hat a, \hat w, \hat b) \leq \frac {\max\{1,R\}Q}{\sqrt{m}} + \,\lambda Q = 2\lambda Q.
\]
In particular,
\[
\sum_{i=1}^m |\hat a_i|\,\big[|\hat w_i|+|\hat b_i|\big] \leq 2Q.
\]

{\bf Step 2.} Note that $\|f\|_{L^\infty((-R,R)^d)} \leq \|f\|_{\B}\max\{1,R\}$ and that $L(z)\leq |z|$. Using the Rademacher risk bound of Lemma \ref{lemma rademacher bound}, we find that with probability at least $1-\delta$, the estimate
\begin{align*}
\Risk(\hat a, \hat w, \hat b) &\leq \widehat \Risk_n(\hat a, \hat w, \hat b) + 2\,\max\{1,R\}\,\sqrt{\frac{\log(2d+2)}n} + 2Q\,\max\{1,R\}\,\sqrt{\frac{\log (2/\delta)}n}\\
	&\leq \frac{2Q\,\max\{1,R\}}{\sqrt{m}} + 2\,\max\{1,R\}\sqrt{\frac{\log(2d+2)}n} + 2Q\,\max\{1,R\}\sqrt{\frac{\log (2/\delta)}n}
\end{align*}
holds.
\end{proof}

\begin{remark}\label{remark a priori hinge}
The rate is $\frac1{\sqrt{m}}+\frac1{\sqrt{n}}$. The rate is worse in $m$ compared to $L^2$-approximation, since $L^1$-estimates were used and $L^1$-loss behaves like a norm whereas $L^2$-loss behaves like the square of a norm. On the other hand, the classification complexity $Q$ (which takes the place of the norm) only enters linearly rather than quadratically. 

While Lemma \ref{lemma low risk smooth} does not apply directly to hinge-loss, a better rate can be obtained directly: Choose $h^*\in \B$ such that $\|h^*\|_\B\leq Q$ and $y_x\cdot h^*(x)\geq 1$ almost everywhere. Use the direct approximation theorem to approximate $2\,h^*$ by $h_m$. Then
\[
\big|L(y_x\cdot h^*(x)) - L(y_x\cdot h_m(x))\big| \leq \big|h_m - h^*\big|^2(x)\qquad\forall\ x\in \spt \P
\]
since $L(y_x\cdot h^*(x)) = L(y_x\cdot h_m(x))$ if $|h_m-h^*|(x)\leq 1$. Considering the differently regularized loss functional
\[\showlabel
\widehat\Risk_{n,\lambda}^*(a,w,b) = \frac1n\sum_{i=1}^n L(-y_x\cdot h_{(a,w,b)}(x)) + \left(\frac\lambda m\sum_{i=1}^m|a_i|\big[|w_i|+|b_i|\big]\right)^2
\] 
for the same $\lambda = \frac{\max\{1,R\}}{\sqrt{m}}$, we obtain the a priori estimate
\[
\Risk(h_{(a,w,b)}) \leq \frac{4Q^2\,\max\{1,R\}^2}{{m}} + 2\,\max\{1,R\}\sqrt{\frac{\log(2d+2)}n} + 4Q^2\,\max\{1,R\}^2\sqrt{\frac{\log (2/\delta)}n}\showlabel
\]
for the empirical risk minimizer with probability at least $1-\delta$ over the choice of training sample $x_1,\dots,x_n$.
\end{remark}

Since loss functions with exponential tails have implicit regularizing properties, they might be preferred over hinge loss. We therefore look for similar estimates for such loss functions.
The exponential function has a convenient tail behaviour, but inconvenient fast growth, which makes it locally but not globally Lipschitz. While this can be handled using the $L^\infty$-version of the direct approximation theorem, the estimates remain highly unsatisfactory due to the dimension-dependent constant.
The loss function $L(y) = \log(1+\exp(y))$ combines the best of both worlds, as both powerful a priori estimates and implicit regularization are available.

The proof is only slightly complicated by the fact that unlike for hinge loss, minimizers do not exist. Note that the risk decay rate $\log m\big[m^{-1/2} + n^{-1/2}\big]$ is almost the same as in the case where minimizers exist, due to the fast decrease of the exponential function.

\begin{theorem}[A priori estimates: log-loss, part I]\label{theorem a priori logistic}
Let $L(z) = \log\big(1+\exp(z)\big)$.
Consider the regularized (empirical) risk functional
\[
\widehat \Risk_{n,\lambda} (a, w,b) = \frac1n\sum_{i=1}^n L\big(- y_{x_i}\,f_{(a,w,b)}(x_i)\big) + \frac{\lambda}m\,\sum_{i=1}^m|a_i|\,\big[|w_i| + |b_i|\big].
\]
where $\lambda = \frac{\max\{1,R\}}{\sqrt{m}}$.
With probability at least $1-\delta$ over the choice of iid data points $x_i$ sampled from $\P$, the minimizer $(\hat a, \hat w, \hat b)$ satisfies
\begin{align*}\showlabel
\Risk (\hat h_m) &\leq  2Q\,\max\{1,R\} \left(1 + \left|\log\left(\frac{2Q\,\max\{1,R\}}{\sqrt m}\right)\right|\right)\cdot\left(\frac1{\sqrt m} + \sqrt{\frac{2\,\log(2/\delta)}n}\right)\\
	&\hspace{7cm} + 2\,\max\{1,R\} \sqrt{\frac{2\,\log(2d+2)}n}.
\end{align*}
{By \eqref{eq risk is proxy}, this implies estimates for the measure of mis-classiified objects.}
\end{theorem}

\begin{proof}
There exists $h^*\in \B$ such that $\|h^*\|_\B\leq Q$ and $y_x\cdot h^*(x)\geq 1$ for $\P$-almost every $x$. With probability $1$ over the choice of points $x_1,\dots, x_n$ we have $\widehat\Risk_n(\mu h^*) \leq L(-\mu) \leq \exp(-\mu)$. In particular,
Using Lemma \ref{lemma low risk smooth} for the empirical risk functional, there exists a finite neural network $h_m$ with $m$ neurons such that 
\[
\|h_m\|_\B\leq Q, \qquad \widehat\Risk_n(\mu h_m) \leq \exp(-\mu) +  \frac{\mu Q\,\max\{1,R\}}{\sqrt m} + \lambda\,\mu Q = \exp(-\mu) + \mu\,\frac{2\,Q\,\max\{1,R\}}{\sqrt{m}}
\]
The optimal value of $\mu = -\log\left(\frac{2Q\,\max\{1,R\}}{\sqrt m}\right)$ is easily found, so if $(\hat a, \hat w, \hat b) \in \R^m\times\R^{md}\times \R^m$ is the minimizer of the regularized risk functional, then
\[
\widehat \Risk_{n,\lambda}(\hat a, \hat w, \hat b) \leq \left(1 + \left|\log\left(\frac{2Q\,\max\{1,R\}}{\sqrt m}\right)\right|\right)\,\frac{2\,Q\,\max\{1,R\}}{\sqrt{m}}.
\]
In particular, using that $\lambda =\frac{\max\{1,R\}}{\sqrt{m}}$ we find that $\hat h_m = h_{(\hat a, \hat w, \hat b)}$ satisfies the norm bound
\[
\|\hat h_m\|_\B \leq \frac{\widehat\Risk_{n,\lambda}(\hat h_m)}\lambda \leq 2Q \left(1 + \left|\log\left(\frac{2Q\,\max\{1,R\}}{\sqrt m}\right)\right|\right).
\]
Since $\|h\|_{C^0}\leq \|h\|_\B\,\max\{1,R\}$, we find that with probability at least $1-\delta$ over the choice of sample points, we have
\begin{align*}
\Risk (\hat h_m) &\leq \widehat \Risk_n(\hat h_m) + 2\,\max\{1,R\}\,\sqrt{\frac{2\,\log(2d+2)}n}\\
	&\hspace{2cm} + \left[2Q\, \left(1 + \left|\log\left(\frac{2Q\,\max\{1,R\}}{\sqrt m}\right)\right|\right)\right]\max\{1,R\}\sqrt{\frac{2\,\log(2/\delta)}n}\\
	&= 2Q\,\max\{1,R\} \left(1 + \left|\log\left(\frac{2Q\,\max\{1,R\}}{\sqrt m}\right)\right|\right)\cdot\left(\frac1{\sqrt m} + \sqrt{\frac{2\,\log(2/\delta)}n}\right)\\
	&\hspace{2cm} + 2\,\max\{1,R\} \sqrt{\frac{2\,\log(2d+2)}n}.
\end{align*}
\end{proof}

We carry out a second argument for different regularization. Note that for $L(z) = \log\big(1+\exp(z)\big)$ we have
\[
L'(z) = \frac{\exp(z)}{1+\exp(z)} = \frac{1}{1+\exp(-z)}, \qquad L''(z) = \frac{\exp(-z)}{(1+\exp(-z))^2} = \frac{1}{\big(e^{z/2} + e^{-z/2}\big)^2}
\]
so $0\leq L''(z) \leq \frac14$ and $\delta_\lambda \leq \exp(-\lambda)$.

\begin{theorem}[A priori estimates: log-loss, part II]\label{theorem a priori logistic two}
Let $L(z) = \log\big(1+\exp(z)\big)$.
Consider the regularized (empirical) risk functional
\[
\widehat \Risk_{n,\lambda}^* (a, w,b) = \frac1n\sum_{i=1}^n L\big(- y_{x_i}\,f_{(a,w,b)}(x_i)\big) + \left(\frac \lambda m\sum_{i=1}^m|a_i|\,\big[|w_i| + |b_i|\big]\right)^2.
\]
where $\lambda = \frac{\max\{1,R\}}{\sqrt{m}}$.
With probability at least $1-\delta$ over the choice of iid data points $x_i$ sampled from $\P$, the minimizer $(\hat a, \hat w, \hat b)$ satisfies
\begin{align*}\showlabel
\Risk (\hat h_m) &\leq \left(1 + 2\,\left|\log\left(\frac{Q^2\,\max\{1,R\}^2}{4m}\right)\right|^2\right)\,\frac{Q^2\,\max\{1,R\}^2}{4m}\\
	&\hspace{2cm} + 2Q\,\max\{1,R\} \left(1 + \left|\log\left(\frac{4Q^2\,\max\{1,R\}^2}m\right)\right|\right) \sqrt{\frac{2\,\log(2/\delta)}n}\\
	&\hspace{2cm} + 2\,\max\{1,R\} \sqrt{\frac{2\,\log(2d+2)}n}.
\end{align*}
{By \eqref{eq risk is proxy}, this implies estimates for the measure of mis-classiified objects.}
\end{theorem}

\begin{proof}
There exists $h^*\in \B$ such that $\|h^*\|_\B\leq Q$ and $y_x\cdot h^*(x)\geq 1$ for $\P$-almost every $x$. With probability $1$ over the choice of points $x_1,\dots, x_n$ we have $\widehat\Risk_n(\mu h^*) \leq L(-\mu) \leq \exp(-\mu)$. In particular,
Using Lemma \ref{lemma low risk Lipschitz}, for given $\mu>0$ there exists a finite neural network $h_m$ with $m$ neurons such that $\|h_m\|_{\B}\leq \mu \big[1+\delta_\mu\big]Q\leq 2\mu\,Q$ and
\[
\widehat\Risk_n(h_m) \leq L(-\mu) + \left[1 + \delta_\mu\right]^2\frac{(\mu Q)^2\max\{1,R\}^2}{8m}
\leq \exp(-\mu) + \frac{(\mu Q)^2\max\{1,R\}^2}{4m}
\]
for sufficiently large $\mu$. We choose specifically $\mu = -\log\left(\frac{Q^2\,\max\{1,R\}^2}{4m}\right)$ to simplify expressions. If $(\hat a, \hat w, \hat b) \in \R^m\times\R^{md}\times \R^m$ is the minimizer of the regularized risk functional, then
\[
\widehat \Risk_{n,\lambda}(\hat a, \hat w, \hat b) \leq \left(1 + 2\,\left|\log\left(\frac{Q^2\,\max\{1,R\}^2}{4m}\right)\right|^2\right)\,\frac{Q^2\,\max\{1,R\}^2}{4m}.
\]
In particular, using that $\lambda =\frac{\max\{1,R\}}{\sqrt{m}}$ we find that $\hat h_m = h_{(\hat a, \hat w, \hat b)}$ satisfies the norm bound
\[
\|\hat h_m\|_\B^2 \leq \frac{\widehat\Risk_{n,\lambda}(\hat h_m)}{\lambda^2} \leq Q^2 \left(1 + 2\,\left|\log\left(\frac{4Q^2\,\max\{1,R\}^2}{m}\right)\right|^2\right).
\]
Since $\|h\|_{C^0}\leq \|h\|_\B\,\max\{1,R\}$, we find that with probability at least $1-\delta$ over the choice of sample points, we have
\begin{align*}
\Risk (\hat h_m) &\leq \widehat \Risk_n(\hat h_m) + 2\,\max\{1,R\}\,\sqrt{\frac{2\,\log(2d+2)}n}\\
	&\hspace{2cm} + \left(1 + 2\,\left|\log\left(\frac{Q^2\,\max\{1,R\}^2}{4m}\right)\right|^2\right)\,\frac{Q^2\,\max\{1,R\}^2}{4m}\\
	&\hspace{2cm} +\left(1 + 2\,\left|\log\left(\frac{Q^2\,\max\{1,R\}^2}{4m}\right)\right|^2\right)\,Q^2\max\{1,R\}^2\, \sqrt{\frac{2\,\log(2/\delta)}n}\\
	&= \left(1 + 2\,\left|\log\left(\frac{Q^2\,\max\{1,R\}^2}{4m}\right)\right|^2\right)\,{Q^2\,\max\{1,R\}^2}\left[\frac1{4m} + \sqrt{\frac{2\,\log(2/\delta)}n}\right]\\
	&\hspace{2cm} + 2\,\max\{1,R\} \sqrt{\frac{2\,\log(2d+2)}n}.
\end{align*}
\end{proof}

\subsection{General hypothesis classes: Multi-layer neural networks and deep residual networks}\label{section multi-layer}

Most parts of the discussion above are not specific to two-layer neural networks and generalize to other function classes. In the proofs, we used the following ingredients. 
 
\begin{enumerate}
\item The fact that every sufficiently smooth function is Barron was necessary to prove that a binary classification problem is solvable using two-layer neural networks if and only if the classes have finite distance. 
\item The direct approximation theorem in Barron space was used to obtain the mostly correct classification result Corollary \ref{corollary mostly correct} and in Section \ref{section a priori} to obtain a priori estimates.
\item The Rademacher complexity of the unit ball of Barron space entered in the a priori estimates to bound the discrepancy between empirical risk and population risk uniformly on the class of functions with bounded Barron norm.
\end{enumerate}

All function classes discussed below are such that Barron space embeds continuously into them, and they embed continuously into the space of Lipschitz-continuous functions. Thus the first of the three properties is trivially satisfied. In particular, the set of solvable classification problems in all function spaces is the same -- the important question is how large the associated classification complexity $Q$ is. For binary classification problems in which the classes have a positive spatial separation, the question of choosing an appropriate architecture really is a question about variance reduction.

We consider a general hypothesis class $\H$ which can be decomposed in two ways: $\H = \bigcup_{Q>0}\H_Q$ (as a union of sets of low complexity $\H_Q = Q\cdot \H_1$) and $\H = \overline{\bigcup_{m=1}^\infty \H^m}$ (as the union of finitely parametrized sub-classes). The closure is to be considered in $L^1(\P)$. Set $\H_{m,Q} = \H^m\cap \H_Q$. We make the following assumptions:

\begin{enumerate}
\item $\P$ is a data distribution on a general measure space $\Omega$.
\item $y_x:\Omega\to\{-1,1\}$ is $\P$-measurable.
\item $\H\subseteq L^2(\P)$ and there exists $h\in \H_Q$ for some $Q>0$ such that $y_x\cdot h(x)\geq 1$ for $\P$-almost all $x\in \Omega$.
\item The hypothesis class $\H$ has the following direct approximation property: If $f\in \H_Q$, then for every $m\in\N$ there exists $f_m\in \H_{m,Q}$ such that $\|f_m-f\|_{L^2(\P)}\leq c_1\,Q\,m^{-1/2}$.
\item The hypothesis class $\H$ satisfies the Rademacher bound $\E_{S\sim\P^n}\Rad(\H_Q;S)\leq c_2\,Q\,n^{-1/2}$.
\item The hypothesis class $\H$ satisfies $\|h\|_{L^\infty(\P)} \leq c_3\,Q$ for all $h\in \H_Q$.
\end{enumerate}

Then Corollary \ref{corollary mostly correct} and Theorems \ref{theorem a priori hinge}, \ref{theorem a priori logistic} and \ref{theorem a priori logistic two} are valid also for $\H$ with slightly modified constants depending on $c_1, c_2, c_3$. Estimates for classes of multi-layer networks can be presented in this form.

From the perspective of approximation theory, multi-layer neural networks are elements of tree-like function spaces, whereas deep (scaled) residual networks  are elements of a flow-induced function space.
Estimates for the Rademacher complexity and the direct approximation theorem in the flow-induced function and tree-like function spaces can be found in \cite{weinan2019lei} and \cite{deep_barron} respectively. 

\begin{remark}
All currently available function spaces describe fully connected neural networks (FNNs) or residual networks with fully connected residual blocks (ResNets). Usually for classification of images, convolutional neural networks (CNNs) are used. The class of convolutional neural networks is a subset of the class of fully connected networks, so all estimates on the function space level equally apply to convolutional networks. We conjecture that sharper estimates should be available for CNNs.
\end{remark}

\section{Multi-label classification}\label{section multi-label}

We now consider the more relevant case of multi-label classification. Many arguments from binary classification carry over directly, and we follow the outline of the previous sections.

\subsection{Preliminaries}

Binary classification problems are particularly convenient mathematically since two identical classes can be easily included on the real line by the labels $\pm 1$. Three or more classes will have different properties with respect to one another since the adjacency relations.

\begin{definition}
A {\em multi-label classification problem} is a collection $(\P, C_1,\dots, C_k,y_1,\dots, y_k)$ where $\P$ is a probability distribution on $\R^d$, $C_1, \dots , C_k \subseteq\R^d$ are disjoint $\P$-measurable sets such that $\P(\bigcup_{j=1}^kC_j) = 1$ and the vectors $y_i\in \R^k$ are the labels of the classes.
\end{definition}

The {\em category function}
\[
y: \R^d\to\R, \qquad y_x = \begin{cases} y_j &x\in C_j\\ 0 &\text{else}\end{cases}\showlabel
\]
is $\P$-measurable. Furthermore, we consider the complementary set-valued {\em excluded category selection}
\[
\Y:\R^d\to (\R^d)^{k-1}, \qquad \Y(x) = \begin{cases}\{y_1,\dots,y_{j-1},y_{j+1},\dots,y_k\} &x\in C_j\text{ for some }j\\ \{0,\dots,0\} &\text{else}\end{cases}.
\]

\begin{definition}
We say that a multi-label classification problem is {\em solvable in a hypothesis class $\H$} of $\P$-measurable functions if there exists $h\in \H$ such that
\begin{equation}
\langle h(x), y_j\rangle \geq \max_{y\in \Y(x)} \langle h(x), y\rangle+1 \qquad\P-\text{almost everywhere}.
\end{equation}
If $\H = \bigcup_{Q>0}\H_Q$, we again set
\begin{equation}
Q_\H(\P,C_1,\dots,C_k) = \inf\{Q' : \text{ $(\P, C_1,\dots, C_k)$ is solvable in $\H_{Q'}$}\}.
\end{equation}
\end{definition}

\begin{remark}
If all classes are assumed to be equally similar (or dissimilar), then $y_i\equiv e_i$. If some classes are more similar than others (and certain misclassifications are worse than others), it is possible to encode these similarities in the choice of the category function $y$ by making the vectors not orthonormal.
\end{remark}

We define $\overline C_j$ as before. 

\begin{lemma}
Let $\mathcal H$ be a hypothesis class of $L$-Lipschitz functions from $\R^d$ to $\R^k$. If a multi-label classification problem is solvable in $\mathcal H$, then
\[
\dist(\overline C_i, \overline C_j) \geq \frac2{L\,|y_i-y_j|}
\]
for all $i\neq j$.
\end{lemma}

\begin{proof}
Let $x_i \in \overline C_i$ and $x_j\in \overline C_j$. Then 
\[
\langle h(x_i), y_i\rangle \geq \langle h(x_j), y_i\rangle +1, \quad \langle h(x_j), y_j\rangle \geq \langle h(x_i), y_j\rangle +1.
\]
so
\begin{align*}
2 &\leq \langle h(x_i) - h(x_j), y_i\rangle - \langle h(x_i) - h(x_j), y_j\rangle\\
	&= \langle h(x_i) - h(x_j), y_i - y_j\rangle\\
	&\leq |h(x_i) - h(x_j)|\,|y_i-y_j|\\
	&\leq L\,|x_i-x_j|\,|y_i-y_j|.
\end{align*}
Taking the infimum over $x_i, x_j$ establishes the result.
\end{proof}

To be able to solve classification problems, we need to assume that for every $i$ there exists $z_i\in \R^d$ such that $\langle z_i, y_i\rangle  > \max_{j\neq i} \langle z_i, y_j\rangle$. Then up to normalization, there exists a set of vectors $\{z_1, \dots, z_k\}\subseteq\R^k$ such that 
\begin{equation}
\langle z_i, y_i\rangle \geq \max_{j\neq i} \langle z_i, y_j\rangle +1.
\end{equation}

\begin{lemma}
We consider the hypothesis class given by Barron space.  Then the classification problem $(\P, C_1, \dots, C_k)$ is solvable if and only if $\delta:= \inf_{i\neq j}\dist(\overline C_i,\overline C_j)>0$ and if $\spt(\P)\subseteq \overline{B_R(0)}$, then
\[
Q_\B(\P,C_1,\dots,C_k,y_1,\dots,y_k) \leq c_d\,\sqrt{k}\,\left(\frac{R+\delta}\delta\right)^d\max_{1\leq i\leq k}|z_i|.
\]
 \end{lemma}

\begin{proof}
The finite separation condition is necessary, since every Barron function is Lipschitz. To prove that it is also sufficient, we proceed like for binary classification.

Set $\bar h(x) = z_i$ if $\dist(x, \overline C_i) < \delta/2$ and $0$ if there exists no $i$ such that $\dist(x, \overline C_i)<\delta/2$. By the definition of $\delta$, this uniquely defines the function. Again, we take the mollification $h:= \eta_{\delta/2}*\bar h$ and see that 
\begin{enumerate}
\item $h(x) = \bar h(x) = z_i$ on $\bar C_i$ for all $i$ and
\item $\|h\|_\B \leq c_d\,\sqrt{k}\,\left(\frac{R+\delta}\delta\right)^d\max_{1\leq i\leq k}|z_i|$.
\end{enumerate}
\end{proof}

\begin{lemma}[Mostly correct classification]
Assume that the multi-label classification problem $(\P, C_1,\dots,C_k, y_1,\dots, y_k)$ has complexity at most $Q$ in Barron space. Then there exists a neural network $h_m(x) = \frac1m\sum_{i=1}^m a_i\,\sigma(w_i^Tx+b_i)$ such that
\[
\|f_m\|_\B\leq Q,\qquad \P\left(\{x : \langle h(x), y_x\rangle \leq \max_{y\in\Y(x)} \langle h(x), y\rangle\}\right)\leq \frac{Q^2}m
\]
\end{lemma}

\subsection{Risk minimization}

While different loss functions are popular in binary classification, almost all works on multi-label classification use {\em cross-entropy loss}. In this setting, the output of a classifier function $h(x)$ is converted into a probability distribution on the set of classes, conditional on $x$. The loss function is the cross-entropy/Kullback-Leibler (KL) divergence of this distribution with the distribution which gives probability $1$ to the true label. The most common way of normalizing the output to a probability distribution leads to a loss function with exponential tails and comparable behavior as in Theorems \ref{theorem a priori logistic} and \ref{theorem a priori logistic two}.

Due to its practical importance, we focus on this case. Other schemes like vector-valued $L^2$-approximation, or approximation of hinge-loss type with loss function
\[
L(x) = \min\big\{ 0, 1 - \min_{y\in \Y(x)} \langle h(x), y_x-y\rangle\big\}
\]
are also possible and lead to behavior resembling results of Theorem \ref{theorem a priori hinge} and Remark \ref{remark a priori hinge}.

Consider the cross-entropy risk functional 
\[
\Risk(h) = -\int_{\R^d} \log\left( \frac{\exp(\langle h(x),y_x\rangle)}{\sum_{i=1}^k\exp(\langle h(x), y_i\rangle)}\right)\,\P(\d x) = \int_{\R^d} L(h(x), y_x)\,\P(\d x)
\]
where $L(z,y) =-\log\left( \frac{\exp(\langle z,y\rangle)}{\sum_{i=1}^k\exp(\langle z, y_i\rangle)}\right)$. This is the most commonly used risk functional in multi-class classification. The quantities 
\[
p_j(z,t) = \frac{\exp(\langle z,y\rangle)}{\sum_{i=1}^k\exp( \langle z, y_i\rangle)}
\]
are interpreted as the probability of the event $y = y_i$ predicted by the model, conditional on $x$. If $\langle z, y_j\rangle > \max_{i\neq j}\langle z, y_i\rangle$ for some $j$, then $\lim_{\lambda\to \infty} p_i(\lambda z, y_j) = \delta_{ij}$ otherwise the probability is shared equally between all categories which have the same (maximal) inner product as $\lambda\to\infty$.

\begin{lemma}[Risk minimization and correct classification]\label{lemma multi-label proxy loss}
Assume that $\Risk(h) < \eps$. Then 
\[\showlabel
\P\left(\bigcup_{i=1}^k \{x\in \overline C_i : \langle h(x), y_i\rangle \leq \max_{j\neq i} \langle h(x), y_j\rangle\}\right) \leq \frac{\eps}{\log 2}.
\]
\end{lemma}

\begin{proof}
If $x\in \overline C_i$ and $\langle h(x), y_i\rangle \leq \max_{j\neq i} \langle h(x), y_j\rangle$, then
\[
\frac{\exp(\langle h(x),y_x\rangle)}{\sum_{i=1}^k\exp(\langle h(x), y_i\rangle)} \leq \frac{\exp(\langle h(x),y_x\rangle)}{\exp(\langle h(x),y_x\rangle)+ \exp(\max_{1\leq j\leq k}\langle  h(x),y_j\rangle)} \leq \frac12
\]
so
\[
L(h(x), x) \geq -\log(1/2) = \log(2).
\]
Thus 
\begin{align*}
\eps &\geq \int_{\R^d}L(h(x),y_x)\,\P(\d x)\\
	&\geq \int_{\bigcup_{i=1}^k \{x\in \overline C_i : \langle h(x), y_i\rangle \leq \max_{j\neq i} \langle h(x), y_j\rangle\}} L(h(x),y_x)\,\P(\d x)\\
	&\geq \int_{\bigcup_{i=1}^k \{x\in \overline C_i : \langle h(x), y_i\rangle \leq \max_{j\neq i} \langle h(x), y_j\rangle\}} \log(2)\,\P(\d x)\\
	&=\log(2)\,\P\left(\bigcup_{i=1}^k \{x\in \overline C_i : \langle h(x), y_i\rangle \leq \max_{j\neq i} \langle h(x), y_j\rangle\}\right).
\end{align*}
\end{proof}

By a similar argument as before, the cross-entropy functional does not have minimizers. We compute the margin functionals on general functions and functions which are classified correctly.

\begin{lemma}[Margin functional] 
Let $\H$ be the unit ball in Barron space. Then
\begin{enumerate} 
\item \[
\lim_{\lambda\to \infty}\frac{\Risk(\lambda h)}\lambda = \int_{\R^d}\max_{1\leq i\leq k} \langle h(x), y_i\rangle - \langle h(x), y_x\rangle \,\P(\d x).
\]
The convergence is uniform over the hypothesis class $\H$.
\item
If $\langle h(x), y_x\rangle \geq \max_{y\in \Y(x)} \langle h(x), y\rangle + \eps$ $\P$-almost everywhere for some $\eps>0$, then 
\[\showlabel
\lim_{\lambda\to \infty} \frac{\log\big(\Risk(\lambda h)\big)}\lambda = -\min_{x\in \spt\P}\left[\langle h(x), y_x\rangle - \max_{y\in \Y(x)}\langle h(x), y\rangle\right].
\]
If $\P$ satisfies the uniform neighbourhood growth condition \eqref{eq uniform growth}, then for any $\eps>0$ the convergence is uniform over the class $\H^c_\eps = \{h\in \H : \langle h(x), y_x\rangle \geq \max_{y\in \Y(x)} \langle h(x), y\rangle + \eps\}$.
\end{enumerate}
\end{lemma}

\begin{proof}
{\bf First claim.}
Note that
\[
\frac{\exp(\lambda\,\langle h(x), y_x\rangle)}{k\,\max_{1\leq j\leq k} \exp(\lambda\,\langle h(x), y_j\rangle)} \leq \frac{\exp(\lambda\,\langle h(x), y\rangle)}{\sum_{j=1}^k\exp(\lambda\,\langle h(x), y_j\rangle)} \leq \frac{\exp(\lambda\,\langle h(x), y\rangle)}{\max_{1\leq j\leq k}\exp(\lambda\,\langle h(x), y_j\rangle)}
\]
so 
\[
\lambda\left[\langle h(x), y_x\rangle - \max_{1\leq j\leq k} \langle h(x), y_j\rangle\right] - \log(k) \leq \log\left(\frac{\exp(\lambda\,\langle h(x), y\rangle)}{\sum_{j=1}^k\exp(\lambda\,\langle h(x), y_j\rangle)}\right) \leq\lambda\left[\langle h(x), y_x\rangle - \max_{1\leq j\leq k} \langle h(x), y_j\rangle\right].
\]
We compute
\begin{align*}
\lim_{\lambda\to\infty} \frac{\Risk(\lambda h)}\lambda &= -\lim_{\lambda\to\infty} \frac1\lambda\int_{\R^d}\log\left( \frac{\exp(\lambda\,\langle h(x),y_x\rangle)}{\sum_{i=1}^k\exp(\lambda\,\langle h(x), y_i\rangle)}\right)\,\P(\d x)\\
	&= \int_{\R^d} \max_{1\leq i\leq k}\langle h(x), y_i\rangle - \langle h(x),y_x\rangle \,\P(\d x)
\end{align*}
The integrand is non-negative, so the functional is minimized if and only if everything is classified correctly.

{\bf Second claim.}
Assume that $\langle h(x), y_x\rangle \geq \max_{y\in \Y(x)} \langle h(x), y\rangle+\eps$ on $\spt\,\P$ for some $\eps>0$. Then
\begin{align*}
-\log\left( \frac{\exp(\lambda\,\langle h(x),y_x\rangle)}{\sum_{i=1}^k\exp(\lambda\,\langle h(x), y_i\rangle)}\right) &= -\log\left(\frac{1}{1+ \sum_{y\in \Y(x)}\exp\big(\lambda\,\big[\langle h(x), y\rangle-\langle h(x), y_x\rangle\big]\big)}\right)\\
	&= \sum_{y\in \Y(x)}\exp\big(\lambda\,\big[\langle h(x), y\rangle-\langle h(x), y_x\rangle\big]\big) + O(\exp(-2\lambda\eps))
\end{align*}
by Taylor expansion. The proof now follows that of Lemma \ref{lemma exponential tail margin}.
\end{proof}

If $\|h\|_\B\leq 1$ and $\lambda\gg 1$, we find that
\[
\Risk(\lambda h) \approx \begin{cases} \lambda\int_{\R^d} \max_{1\leq i\leq k}\langle h(x), y_i\rangle - \langle h(x),y_x\rangle \,\P(\d x) &\text{in general}\\
	e^{-\lambda}\,\exp\left(-\min_{x\in \spt\P}\left[\langle h(x), y_x\rangle - \max_{y\in \Y(x)}\langle h(x), y\rangle\right]\right) &\text{if everything is classified correctly.}
	\end{cases}
\]
Primarily, the functional strives for correct classification with a (very weak) drive towards maximum margin within correct classification. We briefly note that the function
\[\showlabel\label{eq multilabel loss}
L:\R^{k}\times\R^d\to (0,\infty), \qquad L(z,x) = -\log\left(\frac{\exp(\langle y_x, z\rangle)}{\sum_{i=1}^k\exp(\langle y_i, z\rangle)}\right) = \langle y_x, z\rangle - \log\left(\sum_{i=1}^k \exp(\langle y_i, z\rangle)\right)
\]
is Lipschitz-continuous in $z$ since 
\begin{align}
\nabla_zL(z,x) &= y_x - \sum_{i=1}^k \frac{\exp(\langle y_i,z\rangle)}{\sum_{j=1}^k\exp(\langle y_j,z\rangle)}y_i
\end{align}
is uniformly bounded in the $\ell^2$-norm by $\max_{1\leq i \leq k}|y_i|$. To use the direct approximation theorem, only the continuity in the $z$-variables is needed.
The proofs for the following results follow as in Lemma \ref{lemma low risk Lipschitz} and Theorem \ref{theorem a priori logistic}.

\begin{lemma}[Functions of low risk]
\begin{itemize}
\item Assume that $h^*\in \B$. Then for every $m\in \N$ there exists a two-layer network $h_m$ with $m$ neurons such that 
\[
\|h_m\|_\B\leq \|h^*\|_\B\qquad\text{and}\qquad
\Risk(h_m) \leq \Risk(h^*) +  \frac{ \|h^*\|_\B\,\max\{1,R\}\,\max_{1\leq i \leq k}|y_i|}{\sqrt{m}}.
\]
\item
In particular, if $(\P, C_1,\dots, C_k)$ has complexity $\leq Q$ in Barron space, then for any $\lambda>0$ there exists a two-layer network $h_m$ with $m$ neurons such that
\[
\|h_m\|_{\B}\leq \lambda Q, \qquad \Risk(h_m) \leq \exp(-\lambda) + \frac{\lambda Q\,\max\{1,R\}\,\max_{1\leq i \leq k}|y_i|}{\sqrt{m}}
\]
\item Specifying $\lambda = -\log \left(\frac{ Q\,\max\{1,R\}\,\max_{1\leq i \leq k}|y_i|}{\sqrt{m}}\right)$, we find that there exists $h_m$ such that
\begin{align*}
\|h_m\|_{\B} &\leq \left|\log \left(\frac{ Q\,\max\{1,R\}\,\max_{1\leq i \leq k}|y_i|}{\sqrt{m}}\right)\right| Q,\\
\Risk(h_m) &\leq \left[1+ \left|\log \left(\frac{ Q\,\max\{1,R\}\,\max_{1\leq i \leq k}|y_i|}{\sqrt{m}}\right)\right|\right]\,\frac{ Q\,\max\{1,R\}\,\max_{1\leq i \leq k}|y_i|}{\sqrt{m}}.\showlabel
\end{align*}
\end{itemize}
\end{lemma}

Before establishing a priori estimates for multi-label classification with cross-entropy loss, let us recall the following vector-valued version of the `contraction lemma' for Rademacher complexity.

\begin{lemma}\cite[Corollary 1]{maurer2016vector}
Let $S = \{x_1,\dots, x_n\}\subseteq \R^d$, $\H$ be a class of functions $h :\R^d\to\R^k$ and let $G:\R^k\to \R$ have Lipschitz constant $[G]$ with respect to the Euclidean norm. Then
\[
\E\left[ \sup_{h\in \H} \frac1n \sum_{i=1}^n \xi_i\, G(f (x_i))\right]  \leq \sqrt{2}[G]\E\left[\sup_{h\in \H} \frac1n \sum_{i=1}^n\sum_{j=1}^k\xi_{ij}\,f_j (x_i)\right]
\]
where $\xi_i, \xi_{ij}$ are iid Rademacher variables and $f_j (x_i)$ is the $j$-th component of $f(x_i)$.
\end{lemma}
 
In particular
\begin{align*}
\E\left[ \sup_{h\in \H} \frac1n \sum_{i=1}^n \xi_i\, G(f (x_i))\right]  &\leq \sqrt{2}[G]\E\left[\sup_{h\in \H} \frac1n \sum_{i=1}^n\sum_{j=1}^k \xi_{ij}\,f_j (x_i)\right]\\
	&\leq \sqrt{2}[G]\E\left[\sum_{j=1}^k\sup_{h\in \H} \frac1n \sum_{i=1}^n\xi_{ij}\,f_k (x_i)\right]\\
	&= \sqrt{2}\,[G]\,\sum_{j=1}^k\E\left[\sup_{h\in \H} \frac1n \sum_{i=1}^n\xi_{ij}\,f_k (x_i)\right]\\
	&= \sqrt{2}\,[G]\,k \,\Rad(\H; S).
\end{align*}

\begin{theorem}[A priori estimates]\label{theorem a priori cross-entropy lipschitz}
Consider the regularized empirical risk functional
\[
\widehat \Risk_{n,\lambda} (a, w, b) = \frac1n\sum_{i=1}^n L\big(- y_{x_i}\,f_{(a,w,b)}(x_i)\big) + \frac{\lambda}m\,\sum_{i=1}^m|a_i|\,\big[|w_i| + |b_i|\big].
\]
where $\lambda = \frac{\max\{1,R\}}{\sqrt{m}}$.
For any $\delta \in (0, 1)$, with probability at least $1-\delta$ over the choice of iid data points $x_i$ sampled from $\P$, the minimizer $(\hat a, \hat w, \hat b)$ satisfies
\begin{align*}
\Risk (\hat h_m) &\leq  2Q\,\max\{1,R\}\,\max_{1\leq i\leq k}|y_i|\, \left(1 + \left|\log\left(\frac{2Q\,\max\{1,R\}\,\max_{1\leq i\leq k}|y_i|}{\sqrt m}\right)\right|\right)\cdot\left(\frac1{\sqrt m} + \sqrt{\frac{2\,\log(2/\delta)}n}\right)\\\showlabel
	&\hspace{7cm} + 4\,k\max_{1\leq i\leq k}|y_k|\,\max\{1,R\} \sqrt{\frac{2\,\log(2d+2)}n}.
\end{align*}
\end{theorem}

{Due to Lemma \ref{lemma multi-label proxy loss}, this implies a priori estimates also for the measure of mis-classified objects.}

As for binary classification, the rate can be improved at the expense of larger constants. The proofs for the following results follow as in Lemma \ref{lemma low risk smooth} and Theorem \ref{theorem a priori logistic two}.
 We further observe that $L$ in \eqref{eq multilabel loss} is smooth with Hessian
\[
D^2_z L(z,x) = - \sum_{i=1}^k \frac{\exp(\langle y_i,z\rangle)}{\sum_{j=1}^k \exp(\langle y_j,z\rangle)}y_i\otimes y_i + \sum_{i, j=1}^k\frac{\exp(\langle y_i,z\rangle)\,\exp(\langle y_j, z\rangle)}{\left(\sum_{l=1}^k\exp(\langle y_l,z\rangle)\right)^2} y_i\otimes y_j.
\]
In particular, the largest eigenvalue of $D^2_z$ is bounded above by $2\,\max_{1\leq i\leq k} |y_i|^2$. Thus the following hold.

\begin{lemma}[Functions of low risk: Rate improvement]
\begin{itemize}
\item
If $(\P, C_1,\dots, C_k)$ has complexity $\leq Q$ in Barron space, then for any $\lambda>0$ there exists a two-layer network $h_m$ with $m$ neurons such that
\[
\|h_m\|_{\B}\leq \lambda Q, \qquad \Risk(h_m) \leq \exp(-\lambda) + 2\,\max_{1\leq i\leq k} |y_i|^2 \left(\frac{\lambda Q\,\max\{1,R\}}{\sqrt{m}}\right)^2
\]
\item Specifying $\lambda = -\log \left(\frac{ Q^2\,\max\{1,R\}^2\,\max_{1\leq i \leq k}|y_i|}{m}\right)$, we find that there exists $h_m$ such that
\begin{align*}
\|h_m\|_{\B} &\leq \left|\log \left(\frac{ Q^2\,\max\{1,R\}^2\,\max_{1\leq i \leq k}|y_i|^2}{m}\right)\right| Q,\\
\Risk(h_m) &\leq \left[2+ \left|\log \left(\frac{ Q^2\,\max\{1,R\}^2\,\max_{1\leq i \leq k}|y_i|^2}{m}\right)\right|\right]\,\frac{ Q^2\,\max\{1,R\}^2\,\max_{1\leq i \leq k}|y_i|^2}{{m}}.\showlabel
\end{align*}
\end{itemize}
\end{lemma}

\begin{theorem}[A priori estimates: Rate improvement]\label{theorem a priori cross-entropy smooth}
Consider the regularized empirical risk functional
\[
\widehat \Risk_{n,\lambda} (a, w, b) = \frac1n\sum_{i=1}^n L\big(- y_{x_i}\,f_{(a,w,b)}(x_i)\big) + \left(\frac{\lambda}m\,\sum_{i=1}^m|a_i|\,\big[|w_i| + |b_i|\big]\right)^2.
\]
where $\lambda = \frac{\max\{1,R\}}{\sqrt{m}}$.
For any $\delta \in (0, 1)$, with probability at least $1-\delta$ over the choice of iid data points $x_i$ sampled from $\P$, the minimizer $(\hat a, \hat w, \hat b)$ satisfies
\begin{align*}
\Risk (\hat h_m) &\leq 2\left(1 + 2\,\left|\log\left(\frac{Q^2\,\max\{1,R\}^2\max_{1\leq i\leq k}|y_i|^2}{m}\right)\right|^2\right)\,\frac{Q^2\,\max\{1,R\}^2\max_{1\leq i\leq k}|y_i|^2}{m}\\
	&\hspace{1cm} + 2Q\,\max\{1,R\}\,\max_{1\leq i\leq k}|y_i| \left(1 + \left|\log\left(\frac{4Q^2\,\max\{1,R\}^2\max_{1\leq i\leq k}|y_i|^2}m\right)\right| \right)\sqrt{\frac{2\,\log(2/\delta)}n}\\
	&\hspace{1cm} + 4k\,\max\{1,R\} \max_{1\leq i\leq k}|y_i| \sqrt{\frac{2\,\log(2d+2)}n}.\showlabel
\end{align*}
\end{theorem}

{Using \ref{lemma multi-label proxy loss}, we can also show that the measure of mis-classified objects obeys the same a priori bound (up to a factor $\log(2)^{-1}$).}

\section{Problems with infinite complexity}\label{section infinite complexity}

We have shown above that classification problems have finite complexity if and only if the classes have positive distance. This includes many, but not all classification problems of practical importance. To keep things simple, we only discuss binary classification using two-layer neural networks. When classes are allowed to meet, two things determine how difficult mostly correct classification is:

\begin{enumerate}
\item The geometry of the boundary between classes. If, for example, $\overline C_+\cap \overline C_-$ cannot be expressed locally as the level set of a Barron function, then {\em every} two-layer neural network classifier must necessarily misclassify some data points, even in the infinite width limit.

\item The geometry of the data distribution. If most data points are well away from the class boundary, it is easier to classify a lot of data samples correctly than if a lot of data is concentrated by the class boundary. 
\end{enumerate}

To distinguish the two situations, we introduce the following concept.

\begin{definition}
We say that a binary classification problem is weakly solvable in a hypothesis class $\H$ if there exists $h\in \H$ such that $y_x\cdot h(x)>0$ everywhere.
\end{definition}

Note that some sample calculations for margin functionals in Section \ref{section margin} are only valid for strongly solvable classification problems and cannot be salvaged at any expense if a problem fails to be weakly solvable.
Weakly solvable classification problems and concentration at the boundary can be studied easily even in one dimension.

\begin{example}\label{example 1d touching}
Let $\P_\alpha = \frac{\alpha+1}2\cdot |x|^\alpha \cdot \L^1_{(-1,1)}$ be the data distribution with density $|x|^\alpha$ on $(-1,1)$ for $\alpha>-1$, normalized to a probability measure. Assume the two classes are $C_- = (-1,0)$ and $C_+ = (0,1)$. For any of the risk functionals discussed above, the best classifier in the ball of radius $Q$ Barron space is $f_Q(x) = Q\frac x2 = Q \frac{\sigma(x) - \sigma(-x)}2$. We compute
\begin{align*}
\Risk(f_Q) &= \frac{\alpha+1}2\int_{-1}^1L\left(-\frac {Q|x|}2\right)\,|x|^\alpha\dx\\
	&= (\alpha+1)\,\left(\frac Q2\right)^{-(1+\alpha)}\int_{0}^1 L\left(-\frac {Qx}2\right)\,\left(\frac{Qx}2\right)^\alpha\,\frac{Q}2\dx\\
	&= (\alpha+1)\,\left(\frac Q2\right)^{-(1+\alpha)}\int_{0}^{Q/2} L(-z)\,|z|^\alpha \dz\\
	&\sim \left[(\alpha+1)\,\int_{0}^{\infty} L(-z)\,|z|^\alpha \dz\right]\left(\frac 2Q\right)^{\alpha +1}
\end{align*}
for any loss function of the form discussed above. Thus all data points are classified correctly and the risk of any standard functional decays as $Q^{-(\alpha+1)}$ as the norm of the classifier increases -- more slowly the more the closer $\alpha$ is to $-1$, i.e.\ the more the data distributions concentrates at the decision boundary.
\end{example}

Since the risk of classification problems scales the same way in the norm of the classifier independently of which loss function is used, we may consider the mathematically most convenient setting of one-sided $L^2$-approximation since we more easily obtain a $1/m$ error rate in the estimates. 

\begin{definition}
We define the {\em risk decay function}
\[
\rho:(0,\infty)\to [0,\infty), \qquad \rho(Q) = \inf_{\|h\|_\B\leq Q} \int_{\R^d} \max\{0, 1- y_xh(x)\}^2\,\P(\d x).
\]
\end{definition}

By the universal approximation theorem, any continuous function on a compact set can be approximated arbitrarily well by Barron functions \cite{cybenko1989approximation}. For any probability measure $\P$ on the Borel $\sigma$-algebra on $\R^d$, continuous functions lie dense in $L^2(\P)$ \cite[Theorem 2.11]{fonseca2007modern}, so the function $1_{C_+}- 1_{C_-}$ can be approximated arbitrarily well in $L^2(\P)$ by Barron functions. In particular,
\[
\lim_{Q\to\infty} \rho(Q) = \inf_{Q>0}\rho(Q) = \inf_{h\in \B} \int_{\R^d} \max\{0, 1- y_xh(x)\}^2\,\P(\d x) = 0.
\]
The important quantity is the rate of decay of $\rho$. Note that $\rho$ is monotone decreasing and $\rho(Q) = 0$ for some $Q>0$ if and only if the classification problem is (strongly) solvable, i.e.\ if and only if the classes are well separated. Using this formalism, a more general version of Corollary \ref{corollary mostly correct} can be proved. Similar results are obtained for $L^2$-regression in Appendix \ref{appendix regression estimates}.

\begin{lemma}[Mostly correct classification]
Let  $m\in \N$ and $(\P, C_+, C_-)$ a binary classification problem with risk decay function $\rho$. Then there exists a two-layer neural network $h_m$ with $m$ neurons such that
\begin{align}
\P\left(\left\{x\in \R^d : y_x\cdot h_m(x) < 0\right\}\right)& \leq 2\inf_{Q>0} \left[\rho(Q) + \frac{Q^2\,\max\{1,R\}^2}m\right].
\end{align}
\end{lemma}

\begin{proof}
Let $h^Q\in \B$ such that $\|h^Q\|_\B\leq Q$ and 
\[
\rho(Q) = \int_{\R^d} \max\{0, 1- y_xh^Q(x)\}^2\,\P(\d x).
\]
Choose $h_m$ like in the direct approximation theorem. Then
\begin{align*}
\P\left(\left\{x\in \R^d : y_x\cdot h_m(x) < 0\right\}\right)& \leq \int_{\R^d} \max\{0, 1- y_xh_m(x)\}^2\,\P(\d x)\\
	&\leq\int_{\R^d} \max\big\{0, \:1-y_xh^Q(x) + y_x(h^Q- h_m)(x)\big\}^2\,\P(\d x)\\
	&\leq 2\int_{\R^d} \max\{0, 1-y_xh^Q(x)\}^2 + |h^Q- h_m|^2(x)\,\P(\d x).
\end{align*}
\end{proof}

\begin{example}
Assume that $\rho(Q) = c\,Q^{-\gamma}$ for some $\gamma>0$. Then 
\begin{align*}
\bar Q \in \argmin_{Q>0} \left[\rho(Q) + \frac{Q^2\,\max\{1,R\}^2}m\right]&\quad\LRa\quad \rho'(Q) + \frac{2Q\,\max\{1,R\}^2}m = 0\\
	&\quad\LRa\quad -c\gamma\,Q^{-\gamma-2}+ \frac{2\,\max\{1,R\}^2}m =0\\
	&\quad\LRa \quad Q = \left(\frac{2\,\max\{1,R\}^2}{c\gamma\,m}\right)^{\frac1{\gamma+2}}
\end{align*}
and thus 
\begin{align*}
\inf_{Q>0} \left[\rho(Q) + \frac{Q^2}m\right] 
	&= c\left(\frac{2\,\max\{1,R\}^2}{c\gamma\,m}\right)^{\frac\gamma{\gamma+2}} + \left(\frac{2\,\max\{1,R\}^2}{c\gamma\,m}\right)^{-\frac2{\gamma+2}}\frac1m \\
	&= \left[\left(\frac{2\,\max\{1,R\}^2}\gamma\right)^\frac{\gamma}{\gamma+2} + \left(\frac\gamma{2\,\,\max\{1,R\}^2}\right)^\frac{2}{\gamma+2}\right] {c^{\frac2{\gamma+2}}}\,{m^{-\frac\gamma{\gamma+2}}}.
\end{align*}
The correct classification bound deteriorates as $\gamma\to 0$ and asymptotically recovers the bound $m^{-1}$ for strongly solvable problems in the limit $\gamma\to \infty$.
\end{example}

As seen in Example \ref{example 1d touching}, the constant $\gamma$ can be arbitrarily close to zero even in one dimension if the data distribution has large mass close to the decision boundary. Thus we do not expect to be able to prove specific lack of `curse of dimensionality' results since even weakly solvable one-dimensional problems can be very hard to solve. Stricter geometric assumptions need to be imposed to obtain more precise results, like the boundary behaviour of a probability density with respect to Lebesgue measure or the Hausdorff measure on a manifold, or $L^\infty$ bounds.

We present a priori estimates for regularized loss functionals in the setting of Theorem \ref{theorem a priori logistic}. If $\rho$ decays quickly, one can consider a regularized loss functional with quadratic penalty $\widehat\Risk_{n,\lambda}^*$ to obtain a faster rate in $m$ at the expense of larger constants. 

If $\rho$ decays slowly, the linear penalty in $\widehat\Risk_{n,\lambda}$ seems preferable since the rate in $m$ cannot be improved beyond a certain threshold, and the constant in front of the probabilistic term $\sqrt{{\log(2/\delta)}\,{n^{-1}}}$ is smaller.

\begin{lemma}[Functions of low loss]
Assume that $h^*\in \B$ and $L$ has Lipschitz constant $[L]$. 
For any $Q>0$ there exists a two-layer network $h_m$ with $m$ neurons such that
\[
\|h_m\|_{\B}\leq Q, \qquad \Risk(h_m) \leq \inf_{\|h\|_\B\leq Q} \Risk(h) + \frac{[L]\cdot Q\,\max\{1,R\}}{\sqrt{m}}
\]
\end{lemma}

The Lemma is proved exactly like Lemma \ref{lemma low risk Lipschitz}. The quantity $\inf_{\|h\|_\B\leq \lambda Q} \Risk(h)$ is related to $\rho(\lambda Q)$ and the two agree if $L(z) = \max\{0, 1+z\}^2$ (which is not a Lipschitz function).

\begin{theorem}[A priori estimates]\label{theorem a priori unrealizable}
Let $L$ be a loss function with Lipschitz-constant $[L]$. Consider the regularized (empirical) risk functional
\[
\widehat \Risk_{n,\lambda} (a, w,b) = \frac1n\sum_{i=1}^n L\big(- y_{x_i}\,f_{(a,w,b)}(x_i)\big) + \frac{\lambda}m\,\sum_{i=1}^m|a_i|\,\big[|w_i|+|b_i|\big]
\]
where $\lambda = \frac{\max\{1,R\}}{\sqrt m}$.
For any $\delta \in (0, 1)$, with probability at least $1-\delta$ over the choice of iid data points $x_i$ sampled from $\P$, the minimizer $(\hat a, \hat w, \hat b)$ satisfies
\begin{align*}
\Risk(\hat a, \hat w,\hat b)&\leq \inf_{Q>0} \left[\min_{\|h\|_\B\leq Q}\Risk(h) + \frac{2Q\,\max\{1,R\}}{\sqrt{m}}\right] + 2\,\max\{1,R\}\sqrt{\frac{\log(2d+2)}n}\\
	&\qquad + \left[L(0) + [L]\,\inf_{Q>0} \left[\sqrt{m}\,\frac{\min_{\|h\|_\B\leq Q}\Risk(h)}{\max\{1,R\}} + 2Q\right]\right]\,\sqrt{\frac{\log(2/\delta)}n}
\end{align*}
\end{theorem}

A more concrete bound under assumptions on $\rho$ in the case of $L^2$-regression can be found in Appendix \ref{appendix regression estimates}.

\begin{proof}
Let $Q>0$ and choose $h^* \in \argmin_{\|h\|_\B\leq Q}\Risk(h)$. Let $h_m$ be like in the direct approximation theorem. Then
\[
\widehat\Risk_{n,\lambda}(\hat a, \hat w, \hat b) \leq \Risk_{n,\lambda} (a, w,b)  \leq \min_{\|h\|_\B\leq Q}\Risk(h) + 2\lambda\,Q.
\]
We find that
\begin{align*}
\|h_{(\hat a, \hat w, \hat b)}\|_\B &\leq \inf_{Q>0} \left[\sqrt{m}\,\frac{\min_{\|h\|_\B\leq Q}\Risk(h)}{\max\{1,R\}} + 2Q\right], \\
 \widehat \Risk_n(h_{(\hat a, \hat w, \hat b)}) &\leq \inf_{Q>0} \left[\min_{\|h\|_\B\leq Q}\Risk(h) + \frac{2Q\,\max\{1,R\}}{\sqrt{m}}\right].
\end{align*}
Using the Rademacher generalization bound, the result is proved.
\end{proof}

In particular, the Lemma gives insight into the fact that the penalty parameter $\lambda = \frac{\max\{1,R\}}{\sqrt{m}}$ can successfully be chosen independently of unknown quantities like $\P, Q$ and $\rho$.

\begin{remark}
The same strategies as above can of course be used to prove mostly correct classification results or a priori bounds in other function classes or for multi-label classification and different loss functions.
\end{remark}

\section{Concluding Remarks}

We have presented a simple framework for classification problems and given a priori error bounds for regularized risk functionals in the context of two-layer neural networks. The main bounds are given in 
\begin{enumerate}
\item Theorem \ref{theorem a priori hinge} for hinge loss,
\item Theorems \ref{theorem a priori logistic} and \ref{theorem a priori logistic two} for a regularized hinge loss with exponential tail,
\item Theorems \ref{theorem a priori cross-entropy lipschitz} and \ref{theorem a priori cross-entropy smooth} for multi-label classification, and 
\item Theorem \ref{theorem a priori unrealizable} for general Lipschitz-continuous loss functions in binary classification problems of infinite complexity.
\end{enumerate}

The results cover the  most relevant cases, but are not exhaustive. The same techniques can easily be extended to other cases. The extension to neural networks with multiple hidden layers is discussed in Section \ref{section multi-layer}. The very pessimistic complexity bound discussed in Theorem \ref{theorem complexity bound} suggests that there remain hard-to-solve classification problems. A major task is therefore to categorize classification problems which can be solved efficiently using neural networks and obtain complexity estimates for real data sets.

\section*{Acknowledgments}

The authors would like to thank Chao Ma and Lei Wu for helpful discussions. This work was in part supported by a gift to Princeton University from iFlytek.

\appendix
\section{Regression problems with non-Barron target functions}\label{appendix regression estimates}

As suggested in the introduction, also classification problems can be approached by $L^2$-regression where the target function takes discrete values on the support of the data distribution $\P$. If the classes $\overline C_+, \overline C_-$ have a positive spatial distance, the target function coincides with a Barron function on $\spt\,\P$, while the target function fails to be continuous and in particular Barron if the classes touch. 

We extend the analysis of infinite complexity classification problems to the general case of $L^2$-regression. For a target function $f^*\in L^2(\P)$ consider the {\em approximation error decay function}
\[
\rho(Q) = \inf_{\|f\|_\B\leq Q} \|f-f^*\|_{L^2(\P)}^2.
\]

\begin{lemma}[Approximation error decay estimates]
Assume that $f^*$ is a Lipschitz-continuous function on $[0,1]^d$ with Lipschitz constant at most $1$. Then
\begin{enumerate}
\item If $\P$ is the uniform distribution on the unit cube, there exists $f^*$ such that for all $\gamma> \frac{4}{d-2}$ there is a sequence of scales $Q_n\to \infty$ such that $\rho(Q_n) \geq \overline C\,Q_n^{-\gamma}$ for all $n\in\N$.
\item For any distribution $\P$, the estimate $\rho(Q) \leq C_d\,Q^{-2/d}$ holds.
\end{enumerate}
\end{lemma}

\begin{proof}
Without loss of generality, we may assume that $f^*$ is $1$-Lipschitz continuous on $\R^d$. We note that the mollified function $f_\eps = \eta_\eps *f^*$ satisfies
\begin{align*}
\|f_\eps - f^*\|_{L^2(\P)} &\leq \|f_\eps - f^*\|_{L^\infty([0,1]^d)}\\
	&\leq \sup_x \left|f^*(x) - \int_{\R^d} f^*(y)\eta\left(\frac{x-y}\eps\right)\,\eps^{-d}\dy\right|\\
	&\leq \sup_x \int_{\R^d} \big|f^*(x)-f^*(y)\big| \eta\left(\frac{x-y}\eps\right)\,\eps^{-d}\dy\\
	&\leq \eps \int_{\R^d}|z| \eta(z)\,\dz.
\end{align*}
where $c_{d,\eta}$ is very small for large $d$.
On the other hand $\|f_\eps\|_\B\leq C\eps^{-d}$ like in the proof of Theorem \ref{theorem complexity bound}.
\end{proof}

Thus (if the target function is at least Lipschitz continuous), we heuristically expect behaviour like $\rho(Q) \sim Q^{-\alpha}$ for some $\alpha> \frac 2d$.

\begin{lemma}[Approximation by small networks]
Assume that $\spt\,\P\subseteq [-R,R]^d$. 
For every $m\in \N$, there exists a neural network $f_m$ with $m$ neurons such that
\[\showlabel
\|f^*-f_m\|_{L^2(\P)}^2 \leq \inf_{Q>0}2\left[\rho(Q) + \frac{Q^2\,\max\{1,R\}^2}m\right].
\]
\end{lemma}

The proof follows from the direct approximation theorem and the inequality $\|f^*-f_m\|_{L^2}^2 \leq2 \big[ \|f^*-f\|_{L^2}^2 + \|f-f_m\|_{L^2}^2\big]$. In particular, if $\rho(Q) \leq \bar c\,Q^{-\alpha}$, we can optimize $Q = \left(\frac{m\alpha\,\bar c}{2\,\max\{1,R\}^2}\right)^\frac1{2+\alpha}$ to see that there exists $f_m$ such that
\begin{align*}
\frac12\,\|f^*-f_m\|_{L^2(\P)}^2 &\leq \bar c \left(\frac{m\alpha\,\bar c}{2\,\max\{1,R\}^2}\right)^{- \frac\alpha{2+\alpha}} + \frac{\max\{1,R\}^2}m\,\left(\frac{m\alpha\,\bar c}{2\,\max\{1,R\}^2}\right)^\frac2{2+\alpha}\\
	&\leq \left[ \left(\frac 2\alpha\right)^\frac{\alpha}{2+\alpha} + \left(\frac{\alpha}2\right)^\frac2{2+\alpha} \right] \bar c^{\frac2{2+\alpha}}\,\left(\frac{\max\{1,R\}^2}{m}\right)^{\frac{\alpha}{2+\alpha}}.
	\showlabel
\end{align*}

 Let 
\[
\widehat\Risk_{n}(a,w,b) = \frac1{n}\sum_{i=1}^n\big|f^*- f_{(a,w,b)}\big|^2(x_i), \qquad \widehat \Risk_{n,\lambda}(a, w) +  \left[\frac\lambda m\sum_{i=1}^{m} |a_i|\,\|w_i\|\right]^2
\]
for $\lambda = \frac{\max\{1,R\}}{\sqrt m}$. For convenience, we abbreviate 
\[
C_\alpha =2 \left[ \left(\frac 2\alpha\right)^\frac{\alpha}{2+\alpha} + \left(\frac{\alpha}2\right)^\frac2{2+\alpha} \right] \bar c^{\frac2{2+\alpha}}.
\]

\begin{lemma}[A priori estimate]
Assume that $\rho(Q) \leq \bar c\,Q^{-\alpha}$. Then with probability at least $1-\delta$ over the choice of an iid training sample, the minimizer $(\hat a, \hat w, \hat b)$ of the regularized risk functional $\widehat \Risk_{n,\lambda}$ satisfies the a priori error estimate
\begin{align*}
\|f^*- f_{(\hat a, \hat w, \hat b)}\|_{L^2(\P)}^2 &\leq 2\,C_\alpha\,\left(\frac{\max\{1,R\}^2}{m}\right)^{\frac{\alpha}{2+\alpha}}\\
	&+ 4\,C_\alpha\left(\frac{m\alpha\,\bar c}{2\,\max\{1,R\}^2}\right)^\frac1{2+\alpha} \sqrt{\frac{\log(2d+2)}n}\\
	&+4\,C_\alpha^2\left(\frac{m\alpha\,\bar c}{2\,\max\{1,R\}^2}\right)^\frac2{2+\alpha} \max\{1,R\}^2\,\sqrt{\frac{\log(2/\delta)}n}
\end{align*}
\end{lemma}

The proof closely follows that of Theorem \ref{theorem a priori logistic two}: We use a candidate function $f_m$ such that
\[
\|f_m\|_\B\leq \left(\frac{m\alpha\,\bar c}{2\,\max\{1,R\}^2}\right)^\frac1{2+\alpha}, \qquad \|f^*-f_m\|_{L^2(\P_n)}^2 \leq C_\alpha\left(\frac{\max\{1,R\}^2}{m}\right)^{\frac{\alpha}{2+\alpha}}
\]
as an energy competitor, obtain norm bounds for the true minimizer and use Rademacher generalization bounds. Note that the estimate loses meaning as $\alpha\to0$ and that the term $m^{\frac{2}{2+\alpha}}n^{-1/2}$ may not go to zero in the overparametrized regime if $\alpha$ is too small.


\end{document}